%% file: main.tex
\documentclass{article}

\usepackage[T1]{fontenc}
\usepackage{lmodern}
\usepackage{type1cm}
\usepackage{microtype}
\usepackage{graphicx}
\usepackage{subcaption}
\usepackage{booktabs}
\usepackage{hyperref}

\usepackage[accepted]{icml2026}

\usepackage{algorithm}
\usepackage{algorithmic}
\usepackage{bm}
\usepackage{amsmath}
\usepackage{amssymb}
\usepackage{amsthm}
\usepackage{graphicx}
\usepackage{mathtools}
\usepackage{dsfont}

\newcommand{\ind}{\mathds{1}}
\newcommand{\defn}{\coloneqq}
\newcommand{\defnrev}{\eqqcolon}
\newcommand{\var}{\mathsf{Var}}

\newcommand{\cC}{\mathcal{C}}

\newcommand{\cN}{\mathcal{N}}

\newcommand{\bE}{\mathbb{E}}
\newcommand{\bP}{\mathbb{P}}
\newcommand{\bR}{\mathbb{R}}

\newcommand{\diff}{\,\mathrm{d}}

\newcommand{\col}{\mathsf{col}}

\newcommand{\wh}{\widehat}
\newcommand{\wt}{\widetilde}
\newcommand{\clip}{\mathsf{clip}}

\newcommand{\veps}{\varepsilon}

\newcommand{\score}{\mathsf{score}}
\newcommand{\disteq}{\overset{\mathrm d}=}

\newcommand{\polylog}{\mathsf{poly}\log}

\newcommand{\proj}{\mathsf{proj}}

\newcommand{\supp}{\mathsf{supp}}
\newcommand{\low}{\mathsf{low}}

\theoremstyle{plain} 
\newtheorem{lemma}{\bf Lemma} 

\newtheorem{theorem}{\bf Theorem}
 
\newtheorem{claim}{\bf Claim}

\theoremstyle{remark}
\newtheorem{assumption}{\bf Assumption}

\newtheorem{remark}{Remark}

\icmltitlerunning{
  Diffusion Models Are Statistically Optimal for Learning Low-Dimensional Multi-Modal Distributions
  }

\begin{document}

\twocolumn[
  \icmltitle{
    Diffusion Models Are Statistically Optimal for Learning \\Low-Dimensional Multi-Modal Distributions
    }

  \icmlsetsymbol{equal}{*}

  \begin{icmlauthorlist}
    \icmlauthor{Jingda Wu}{yyy}
    \icmlauthor{Changxiao Cai}{yyy}
    
  \end{icmlauthorlist}

  \icmlaffiliation{yyy}{Department of Industrial and Operations Engineering, University of Michigan, Ann Arbor, USA}

  \icmlcorrespondingauthor{Changxiao Cai}{cxcai@umich.edu}
  
  \icmlkeywords{Diffusion models, Sample complexity, Low-dimensional structure, Minimax optimality}

  \vskip 0.3in
]

\printAffiliationsAndNotice{}

\input{abstract}
\input{introduction}
\input{problem_formulation}

\input{Results}

\input{Proof_Overview}
\input{Numerical_results}

\input{discussions}

\section*{Acknowledgements}
C.\ Cai is supported in part by the NSF grant DMS-2515333.

\section*{Impact Statement}
This paper presents work whose goal is to advance the field of diffusion model theory. There are many potential societal consequences of our work, none of which we feel must be specifically highlighted here.

\bibliographystyle{icml2026}
\bibliography{reference}

\appendix
\onecolumn
\input{pf-of-theorems}
\input{pf-of-lemmas}
\input{pf-of-claims}
\input{Auxiliary_lemmas}

\end{document}

%% file: abstract.tex
\begin{abstract}
Score-based diffusion models have demonstrated remarkable empirical success in learning high-dimensional distributions, particularly those exhibiting low-dimensional and multi-modal structures. However, theoretical understanding of their statistical efficiency remains limited. Existing theories typically rely on strong regularity assumptions, such as uniformly bounded densities or globally smooth score functions, which fail to capture such intrinsic structures.
In this work, we study the sample complexity of diffusion models for learning distributions supported on a union of low-dimensional subspaces. Assuming that the data distribution within each subspace is subgaussian, we show that diffusion models require at most the order of $\widetilde{O}(\varepsilon^{-k \vee 2})$ (up to some logarithmic factor) samples to achieve $\varepsilon$ sampling error in 1-Wasserstein distance, where $k$ is the intrinsic dimension.
This near-optimal convergence rate depends only on the intrinsic dimension and significantly improves upon prior theoretical guarantees that suffer from the curse of dimensionality. Notably, our analysis applies to a broad collection of distributions without imposing smoothness, bounded-density, or log-concavity assumptions.
Overall, our results show that diffusion models can statistically adapt to intrinsic low-dimensional structure while naturally accommodating multi-modal data, offering a rigorous theoretical justification for their success in complex high-dimensional learning tasks.
\end{abstract}

%% file: introduction.tex
\section{Introduction}
Score-based diffusion models \citep{sohldickstein2015deepunsupervisedlearningusing,song2020generativemodelingestimatinggradients} have achieved state-of-the-art performance across a wide range of generative modeling applications, including image and video generation (\citealp{ho2020denoisingdiffusionprobabilisticmodels}; \citealp{ho2022videodiffusionmodels}), signal processing \citep{song2022solvinginverseproblemsmedical}, and language modeling \citep{austin2023structureddenoisingdiffusionmodels,nie2025largelanguagediffusionmodels}.
At a high level, diffusion models generate samples by starting from Gaussian noise and iteratively denoising via a learned reverse-time diffusion dynamics. This procedure relies on accurate estimation of the \textit{score function} (the gradient of the log-density) along a forward noising process.

Diffusion models can be viewed as unsupervised distribution learners---given finite training samples from an unknown data distribution, they aim to generate new samples that faithfully follow the same law. 
This perspective raises a fundamental statistical question: 
\textit{how many training samples are needed for diffusion-based sampling to accurately learn the underlying data distribution, and can this sample complexity match the information-theoretic limit?}
At a conceptual level, diffusion sampling is inherently two-stage: it first uses training data to estimate the time-indexed scores along a forward diffusion process, and then plugs these learned scores into an iterative sampling procedure to generate an output.
Therefore, addressing the above question calls for sample complexity guarantees that jointly control 
both the score estimation error and the error accumulated during sampling.

\paragraph{Leveraging intrinsic structures.}
Motivated by this, a growing body of statistical theory has been developed to understand the sample complexity of diffusion models \citep{shah2023learning,oko2023diffusionmodelsminimaxoptimal,chen2024learning,cole2024score,li2024good,dou2024optimalscorematchingoptimal}.

For a broad class of $d$-dimensional distributions with $\beta$-H\"older smooth densities (without assuming smooth scores or log-concave/uniformly-bounded densities), state-of-the-art theory shows that both DDPM \citep{zhang2024minimaxoptimalityscorebaseddiffusion} and DDIM \citep{cai2025minimaxoptimalityprobabilityflow} require on the order of  (up to logarithmic factors)
\begin{align}\label{eq:sample complexity general}
\veps^{-\frac{d+2\beta}{\beta}}
\end{align}
training samples to generate an output within $\veps$ total variation (TV) distance to the target distribution.
While this sample complexity is (nearly-)minimax optimal for general smooth-density classes, it suffers from the curse of dimensionality as the ambient dimension $d$ grows. Consequently, such guarantees fail to fully explain the empirical effectiveness of diffusion models in modern high-dimensional applications, suggesting that the worst-case bounds in \eqref{eq:sample complexity general} may be overly pessimistic for structured distributions arising in practice.

To narrow this gap, recent work has explored whether, and in what sense, diffusion models can exploit intrinsic structures underlying the data distributions. 
However, statistical theory of diffusion models for structured data distributions remains far from complete.
Existing results in this direction typically focus on distributions supported on a \textit{single} low-dimensional structure, such as a linear subspace or manifold \citep{chen2023scoreapproximationestimationdistribution,tang24a,azangulov2025convergencediffusionmodelsmanifold,yakovlev2025generalization}, a factor model \citep{chen2025diffusion}, or certain dependence structures \citep{fan2025optimalestimationfactorizabledensity}.
Although these works establish improved sample complexities governed by intrinsic rather than ambient dimension, they require strong assumptions on the data distribution. A prominent example is the requirement for the density to be uniformly bounded away from zero on its support.
While this condition is standard in the nonparametric statistics literature and technically convenient, it excludes important multi-modal structures with well-separated components, where the density necessarily becomes small, or even vanishes, between modes.

More fundamentally, the prevailing ``single manifold/subspace'' paradigm limits our theoretical understanding of diffusion models' capabilities to 
learn heterogeneous distributions whose different modes concentrate near different low-dimensional structures. Such geometry is common in modern high-dimensional data, where distinct classes or clusters may occupy separate subspaces or manifolds \citep{Vidal2010ATO,brown2022verifying}.
As a result, existing theory still falls short of explaining the empirical effectiveness of diffusion models when learning low-dimensional, multi-modal distributions.

\subsection{Main contributions}
In this paper, we develop a statistical theory for diffusion models, aimed at understanding how many samples are required to learn low-dimensional, multi-modal distributions.

Concretely, we consider a target data distribution $p^\star$ supported on a union of subspaces (UoS), i.e., 
\begin{align*}
    \supp (p^\star) \subseteq \cup _{i=1}^M V_i,     
\end{align*}
where each $V_i$ is a linear subspace with dimension $k_i$. We denote by $k \defn \max_{i\in[M]} k_i$ the maximum intrinsic dimension. In addition, we assume that the restriction of the target distribution to each subspace is $\sigma$-subgaussian.

Under these two assumptions, we construct a kernel-based regularized score estimator $\wh s_t$ for the score function of the Gaussian-smoothed distribution $p_t \defn p^\star \ast \cN(0,tI_d)$ for any $t>0$.
Given $n$ samples drawn from the target distribution $p^\star$, we establish a finite-sample $L^2$ score estimation error bound:
\begin{align*}
    \mathop{\bE}\big[\|\widehat{s}_t(X)-s_t^\star(X)\|_2^2\big]
    = \wt O\bigg(\frac{1}{n}
    \Big( \frac{1}{t}+\frac{1}{t^{{(k\vee2)}/2+1}}\Big) \bigg).
\end{align*}
Here the expectation is taken over both the training data and $X\sim p_t$, where we write $a\vee b \defn \max\{a,b\}$.

Building on this score estimation guarantee, we prove that diffusion samplers require at most the order of (up to logarithmic factors)
\begin{align*}
 \veps^{- (k \vee 2)}
\end{align*}
training samples to generate a sample that is $\veps$-close in 1-Wasserstein distance to the target distribution $p^\star$. 
Importantly, this convergence rate depends only on the intrinsic dimension $k$, rather than the ambient dimension $d$, and matches the minimax optimal rate for learning a $k$-dimensional distribution \citep{chewi2024statisticaloptimaltransport}.
Moreover, our theory requires only subgaussian tails on each subspace, without imposing any restrictive assumptions on scores or densities. As a result, our framework naturally accommodates multi-modal distributions with well-separated components.

Finally, we emphasize that the kernel-based score estimator developed in this paper is primarily a theoretical proof device, rather than a practical alternative to neural network (NN)-based score estimation. Nevertheless, our results provide an important step toward statistical guarantees for NN score-based diffusion models. 
In particular, they establish the achievability of the fundamental statistical limit and identify the structural properties that analysis of NN score estimators should capture, while also providing an explicit low-dimensional target for NN approximation. More discussion on extensions to NN-based scores can be found in Section \ref{sec:discussion}.

\subsection{Related works}
\label{sec:related works}

\paragraph*{Statistical theory for diffusion models.}
Recent work has begun to provide finite-sample guarantees for diffusion-based sampling by studying statistical bounds for score estimation error.
Under strong density regularity assumptions (e.g., boundedness on compact domains), \citet{oko2023diffusionmodelsminimaxoptimal} showed that neural-network-based ERM score estimators lead to minimax-optimal rates in both TV and $W_1$ distances when used with reverse SDE samplers.
Using nonparametric constructions under density lower-bound conditions, \citet{dou2024optimalscorematchingoptimal} derived minimax-optimal score estimation rates and corresponding sampling guarantees.
More recently, for subgaussian targets with $\beta$-H\"older smooth densities, \citet{zhang2024minimaxoptimalityscorebaseddiffusion} established minimax-optimality for DDPM using truncated kernel score estimators, and \citet{cai2025minimaxoptimalityprobabilityflow} obtained an end-to-end minimax-optimal convergence analysis for ODE-based diffusion (DDIM/probability flow) by combining smoothed score estimation with convergence analysis of the sampling dynamics.

\paragraph*{Adaptation to low-dimensional structures.}
For distributions supported on a linear subspace, \citet{chen2023scoreapproximationestimationdistribution} established convergence rates governed by the subspace dimension under smooth score assumptions. 
For manifold-supported targets, \citet{azangulov2025convergencediffusionmodelsmanifold} proved analogous intrinsic-dimension rates, but the analysis requires controlling geometric approximation error (e.g., via Hausdorff distance) and typically relies on density lower-bound conditions on the support.
Beyond geometric support constraints, \citet{fan2025optimalestimationfactorizabledensity} obtained minimax-optimal rates for diffusion learning under structured dependence (exponential-interaction) models.
In addition, \citet{wang2025diffusionmodelslearnlowdimensional} analyzed mixtures of low-rank Gaussians, focusing on the special case of orthogonal subspaces. \citet{boffi2025shallow} showed that shallow NN-based diffusion models can provably adapt to hidden low-dimensional subspace structure under independent component data models and smoothness assumptions on the latent scores.

Complementary to the statistical perspective, a parallel line of work studies whether the sampling stage of diffusion models can automatically exploit low-dimensional data structure \citep{li2024adaptingunknownlowdimensionalstructures,liang2025lowdimensionaladaptationdiffusionmodels,potaptchik2024linear,huang2024denoising}. These works show that the iteration complexity, the number of sampling iterations required to achieve a desired accuracy, also depends only on the intrinsic dimension rather than the ambient dimension. In addition, low-dimensional adaptation has also been investigated for discrete diffusion models when learning discrete distributions \citep{li2025breaking,zhao2026adaptation,cai2026confidence,chen2025optimal,dmitriev2026efficient}.

\subsection{Notation}
For $a,b\in\bR$, we denote $a \vee b \defn \max\{a,b\}$ and $a\wedge b \defn \min\{a,b\}$.
For positive integer $M$, let $[M]\defn\{1,\cdots,M\}$. 
For random vector $X$, we use $p_X$ to denote its distribution or probability density function, depending on the context.
For any vector $x\in \mathbb{R}^d$, we denote $\| \cdot \|_p$ as its $p$-norm, i.e., $\|x\|_p \defn (\sum_{i=1}^d |x_i|^p)^{1/p}$, and write $\|x\|_{\infty} \defn \max_{i}|x_i|$. We use $\|\cdot\|$ to denote the 2-norm for simplicity. 
For any vector $x\in \mathbb{R}^d$ and any $i,j\in[d]$ with $i<j$, we denote by $x_{i:j} \in \mathbb{R}^{j-i+1}$ the subvector consisting of the $i$-th through $j$-th entries of $x$.
For vectors $\{\alpha_i\}_{i=1}^k$, we denote by $\mathsf{span}(\{ \alpha_i \}_{i=1}^k) $ the linear space spanned by these vectors.

For a probability distribution $p$ and a random vector $X$, we write $X \sim p$ to mean that $X$ follows the distribution $p$. We denote by $\supp(p)$ the support of probability measure $p$, i.e., the smallest closed set $S$ such that $p(S)=1$.
In addition, let $\mathds{1}\{ \cdot \}$ denote the indicator function. 
For random vectors $X, Y$,
we define the 1-Wasserstein distance between their distributions $p_X$ and $p_Y$ by
\begin{align*}
    W_1(p_X,p_Y) \defn \inf_{\gamma \in \Gamma(p_X,p_Y)}\iint \|x-y\| \, \gamma(\!\diff x,\!\diff y),
\end{align*}
where $\Gamma(p_{X},p_{Y})$ denotes the set of couplings of $p_X$ and $p_Y$. 
For probability distributions $P,Q$, we denote their convolution by $P\ast Q$.
Finally, we use $\mathsf{poly}(n)$ to denote a polynomial function of $n$ where the specific degree may vary across different contexts.

%% file: problem_formulation.tex
\section{Problem formulation}
\label{sec:problem formulation}
\subsection{Preliminaries}
In this section, we briefly introduce the score-based diffusion models.

\paragraph*{Forward process.}
The forward process starts from the target distribution $p^\star$ and gradually adds Gaussian noise.
A popular choice is the Ornstein-Uhlenbeck (OU) process \cite{song2021scorebasedgenerativemodelingstochastic}:
\begin{align}
    \diff X_t = -X_t \diff t + \sqrt{2} \diff B_t, \quad \text{with}\quad X_0 \sim p^\star.  \label{eq:forward SDE}
\end{align}
Here $(B_t)_{t\in [0,T]}$ is a standard Brownian motion in $\mathbb{R}^d$. A key property of this OU process is that the conditional distribution of $X_t$ given $X_0$ remains Gaussian for all $t$. More precisely, one can verify that
\begin{align}
    X_t \mid X_0 \disteq c_t X_0 + \sigma_t W_t \label{eq:forward conditional dist}
\end{align}
where $c_t \defn e^{-t}$, $\sigma_t \defn \sqrt{1-e^{-2t}}$, and $W_t \sim \mathcal{N}(0,I_d)$ is independent of $X_0$. 
In particular, the parameter $t$ fully determines the noise level of the forward process. As $t$ becomes sufficiently large, $c_t$ approaches zero and the distribution of $X_t$ becomes close to the standard Gaussian distribution $\mathcal{N}(0,I_d)$.

\paragraph*{Reverse process.}
Running the forward dynamics backward in time transforms Gaussian noise into samples from $p^\star$, forming the basis of diffusion-based sampling. For the OU process in \eqref{eq:forward SDE}, its time-reversal SDE is given by
\begin{align}
    Y_0 &\sim p_{X_T}, \notag\\
    \diff Y_t &= \big(Y_t + 2\nabla \log p_{X_{T-t}}(Y_t)\big) \diff t + \sqrt{2} \diff \overline{B}_t.
    \label{eq:reverse SDE}
\end{align}
Here $p_{X_{T-t}}$ denotes the density of the forward process (\ref{eq:forward SDE}) at time $T-t$ and $\{ \overline{B}_t\}_{t\in [0,T]}$ is a standard Brownian motion in $\mathbb{R}^d$. By classical time-reversal results for SDEs \citep{ANDERSON1982313}, this process satisfies $Y_{T-t}\disteq X_t$ for all $t\in[0,T]$.

The crucial ingredient in the reverse dynamics is the score function of the marginal distributions of the forward process. For a random vector $X \in \mathbb{R}^d$ with density $p_X$, its \textit{score function} is given by
\begin{align}
    s^\star_{X}(x) \defn \nabla \log p_{X}(x) = \frac{\nabla p_{X}(x)}{p_{X}(x)}. \label{eq:score function}
\end{align}
Since these scores are unknown in practice, they must be estimated from training samples $\{X^{(i)}\}_{i=1}^n$ drawn from $p^\star$.

\paragraph*{Sampling procedure.}
Since $X_T\overset{\mathrm d}{\to}\mathcal{N}(0,I_d)$ as $T\to\infty$, diffusion-based sampling can be implemented by initializing the reverse dynamics from $\mathcal{N}(0,I_d)$ and replacing the true score $s_{X_t}^\star$ with a learned estimator $\widehat{s}_{X_t}$. 
The resulting procedure is summarized in Algorithm~\ref{alg:reverse SDE sampling}. 
We introduce an early stopping time $\tau>0$ to avoid the small-time regime, where score estimation is most challenging. In practice, the reverse SDE can be implemented using numerical methods such as Euler-Maruyama.

\begin{algorithm}[t]
\caption{Sampling procedure}
\label{alg:reverse SDE sampling}
\begin{algorithmic}[1]
\STATE \textbf{Input:} Early stopping time $\tau >0$, end time $T>0$, score estimator $\widehat{s}_{X_t}$ for $t\in [\tau,T]$.
\STATE Sample $ y \sim \mathcal{N}(0,I_d)$.
\STATE Solve the reverse SDE:
\begin{align}
     \diff \widehat{Y}_t = \big(\widehat{Y}_t + 2\widehat{s}_{X_{T-t}}(\widehat{Y}_t)\big) \diff t + \sqrt{2} \diff B_t \label{eq:reverse SDE in alg}
\end{align}
for $t\in [0,T-\tau]$ with $\widehat{Y}_0 = y$.
\STATE \textbf{Output:} generated sample $\widehat{Y}_{T-\tau}$.
\end{algorithmic}    
\end{algorithm}

\paragraph*{Score estimation reduction.}
To estimate the score function $s^\star_{X_t}$, it is often more convenient to construct score estimator for the following variance-exploding (VE) process
\begin{align}
    \diff Z_t = \diff B_t, \quad  \text{with } Z_0 \sim p^\star. \label{eq:VE process}
\end{align}
Here $\{B_t\}_{t\geq 0}$ also denotes the Brownian motion and thus $Z_t$ follows the distribution $p^\star  \ast \mathcal{N}(0,tI_d)$. 
It is straightforward to verify that the score functions of $X_t$ and $Z_t$ satisfy
\begin{align}
    s^\star_{X_t}(x) = \frac{1}{c_t} s^\star_{Z_{h(t)}} (\frac{x}{c_t}) \quad \text{with}\quad h(t)\defn  \frac{\sigma_t^2}{c_t^2}. \label{eq:score between VE process and VP} 
\end{align}
As a result, it suffices to estimate the score function of $Z_t$ for any $t > 0$ and then define $\widehat{s}_{X_t}(x)\defn \frac{1}{c_t} \widehat{s}_{Z_{h(t)}}(\frac{x}{c_t})$ as the estimator of $s^\star_{X_t}(x)$. 
For notational simplicity, we denote by $s_t^\star \defn \nabla \log p_{Z_t}$ the score function of $Z_t$, and let $\widehat{s}_t$ denote its estimator.
The derivation of (\ref{eq:score between VE process and VP}) is provided in Appendix~\ref{pf-of-thm:thm:high prob TV}.

\subsection{Assumptions}
In this section, we introduce the assumptions imposed on the target distribution $p^\star$.

First, to capture low-dimensional, multi-modal structure, we assume that the support of $p^\star$ is contained in a finite union of low-dimensional linear subspaces.
\begin{assumption}[Union of low-dimensional subspaces]
    \label{assume:multi-modal}
    There exist linear subspaces $V_1, V_2, \ldots, V_M \subseteq \mathbb{R}^d$, with dimension $\mathsf{dim}(V_i) = k_i$, such that 
    \begin{align*}
        \supp(p^\star) \subseteq \cup_{i=1}^M V_i.
    \end{align*}
    Moreover, $p^\star$ assigns zero probability to intersections between different subspaces, i.e.,
    \begin{align}
        p^\star(V_i \cap V_j) = 0, \quad \forall i\neq j. \label{eq:separation}
    \end{align}
    Finally, each subspace has non-trivial mass:
    \begin{align}
        p^\star(V_i) \geq \frac{1}{c_{p}M}, \quad \forall i\in [M]
    \end{align}
    for some constant $c_{p}>0$. 
\end{assumption}

This assumption provides a tractable abstraction for low-dimensional, multi-modal distributions, where different modes may concentrate on different subspaces. Such union-of-subspaces structure has been widely used in the modeling of heterogeneous high-dimensional data \citep{wang2025diffusionmodelslearnlowdimensional} and has also been observed empirically in real-world datasets \citep{brown2022verifying,kamkari2024geometric}.

For each $i\in [M]$, let $p_i^\star \defn p^\star \mid_{V_i}$ denote the restriction of the target distribution $p^\star$ to subspace $V_i$. By Assumption~\ref{assume:multi-modal}, we can decompose the target distribution as $p^\star = \sum_{i=1}^M p_i^\star$.

For each subspace $V_i$, let $A_i \in \mathbb{R}^{d\times k_i}$ be a matrix whose columns form an orthogonal basis of $V_i$: 
\begin{align*}
    V_i = \mathsf{span}(\col(A_i)), \quad A_i^{\top}A_i = I_{k_i}.
\end{align*}
Denote by $\proj_i:\mathbb{R}^d \rightarrow V_i$ the projection onto $V_i$, given by
\begin{align*}
    \proj_i(x) = A_iA_i^{\top}x.
\end{align*}
\begin{remark}
    Our framework can be extended naturally to distributions concentrated near a union of low-dimensional subspaces. In this paper, we focus on the noiseless setting to isolate the essential roles of low-dimensional structure and multi-modality, without introducing the additional technical complications caused by ambient noise. Extensions to noisy settings are discussed in Section~\ref{sec:discussion}.
\end{remark}

Next, we impose a mild subgaussian assumption on the target distribution within each subspace.
\begin{assumption}[Subgaussian within each subspace]
    \label{assump:sub-gaussian target}
    Let $p_i^\low$ be the normalized push-forward distribution of $p_i^\star$ onto $\mathbb{R}^{k_i}$ under $A_i^{\top}$: 
    \begin{align}
        p_i^\low \defn \mathsf{law}(A_i^{\top}Z), \quad Z \sim \frac{p_i^\star}{p^\star(V_i)}. \label{eq:dist in each low-dim}
    \end{align}
    We assume that $p_i^\low$ is $\sigma_i$-subgaussian, that is, for any unit vector $\theta \in \mathbb{R}^{k_i}$ with $\|\theta\|_2 = 1$,
    \begin{align*}
        \mathbb{E} \bigl[\exp\bigl( (X^{\top}\theta/\sigma_i)^2\bigr)\bigr] \leq 2, \quad X \sim p_i^\low.
    \end{align*}
    We denote $\sigma \defn \max_{i\in [M]} \sigma_i$.
\end{assumption}

The subgaussian assumption is fairly mild in the sense that it subsumes any distribution with bounded support, which covers a wide range of practical data such as image data.

%% file: Results.tex
\section{Main results}
\label{sec:theoretical results}
This section introduces our score estimator and presents theoretical guarantees for both score estimation and sampling.

\subsection{Algorithm}
Given $n$ i.i.d.\ samples $\{ X^{(i)}\}_{i=1}^n$ drawn from the target distribution $p^\star$, our goal is to build a score estimator $\widehat{s}_t$ that learns the score function $s_t^\star$ of the perturbed data distribution  $p_t = p^\star \ast \cN(0,tI_d)$ for any time $t>0$.

\paragraph{Motivation.}
Observe that the density $p_t$ can be written as
\begin{align*}
    p_t(x) &= \big(p^\star \ast \cN(0, t I_d)\big)(x) =  \int \varphi_t(x-y;d) p^\star(\!\diff y),
\end{align*}
where $\varphi_t(x;d) \defn (2\pi t)^{-d/2}\exp\bigl(-\|x\|_2^2/(2t)\bigr)$ is the density of $\mathcal{N}(0,tI_d)$ in $\bR^d$.
Recall that under the UoS assumption, we can decompose $p^\star = \sum_{i=1}^M p_i^\star$, yielding
\begin{align*}
    p_t(x) &= \sum_{i=1}^M \int_{V_i} \varphi_t(x-y;d) p_i^\star( \!\diff y). 
\end{align*}
The gradient of $p_t$ can then be computed as
$$ \nabla p_t(x) = \sum_{i=1}^M \int_{V_i} -\frac{x-y}{t} \varphi_t(x-y;d) p_i^\star( \!\diff y) .$$
Therefore, the score function $s_t^\star = {\nabla p_t}/{p_t}$ of $p_t$ admits the following mixture-type decomposition:
\begin{subequations}\label{eq:score decomp}
\begin{align}
    s_t^\star(x) &= \sum_{i=1}^M \frac1{p_t(x)}\int_{V_i} -\frac{x-y}{t} \varphi_t(x-y;d) p_i^\star(\!\diff y) \notag\\
    & \defnrev \sum_{i=1}^M w_t(i,x) \cdot s_t(i,x),
\end{align}
where we define
\begin{align}
    w_t(i,x) &\defn \frac{\int_{V_i} \varphi_t(x-y;d) p_i^\star(\!\diff y)}{p_t(x)}\defnrev \frac{q_t(i,x)}{p_t(x)}, \label{eq:defn w_t(i,x)} \\
    s_t(i,x)&\defn \frac{\int_{V_i} -\frac{x-y}{t} \varphi_t(x-y;d) p_i^\star(\!\diff y)}{\int_{V_i} \varphi_t(x-y;d) p_i^\star(\!\diff y)} \label{eq:defn s_t(i,x)}.
\end{align}
\end{subequations}
Intuitively, $w_t(i,x)$ can be interpreted as the posterior probability (or effective mixture weight) that $x$ originates from the $i$-th subspace after Gaussian smoothing, while $s_t(i,x)$ is the score of the corresponding smoothed component.

Our key observation is that each score component $s_t(i,x) \in \bR^d$ admits a favorable normal-tangent decomposition. 
The normal part is essentially the score of a time-dependent Gaussian distribution and has a closed-form expression, while the tangent component is determined entirely by a $k_i$-dimensional score function \citep{chen2023scoreapproximationestimationdistribution}:
\begin{align}
    s_t(i,x) = -\frac{1}{t}\big(x-\proj_i(x)\big) + A_i s_t^\low(i,A_i^{\top}x), \label{eq:s_t(i,x) decomp}
\end{align}
where $\proj_i(x)\defn  A_i A_i^{\top} x$ is the projection of $x$ onto subspace $V_i$, and $s_t^\low(i,\cdot):\mathbb{R}^{k_i}\rightarrow \mathbb{R}^{k_i}$ is the score function of the $k_i$-dimensional smoothed distribution $p^\low_i \ast \cN(0,tI_{k_i})$ (see \eqref{eq:dist in each low-dim}) on the subspace $V_i$.
This decomposition reduces score estimation to a low-dimensional problem: once $V_i$ is identified, estimating $s_t(i,\cdot)$ is governed by the difficulty of estimating the low-dimensional score $s_t^\low(i,\cdot)$ in dimension $k_i$, rather than the ambient dimension $d$.

Motivated by this observation, we propose a two-step score estimation procedure.
We first use the data to estimate the subspaces $\{A_i\}_{i=1}^M$ and construct a classification function $c:\mathbb{R}^d\to[M]$ that assigns points to subspaces (such that $c(x)=i$ if and only if $x\in V_i$).
Given these estimates, we then estimate the component scores and mixture weights, and combine them to form the full score estimator.
For theoretical clarity, we employ sample splitting, where $n_0$ samples are used for subspace recovery and the remaining $N = n-n_0$ samples are used for score estimation.

In what follows, we describe the proposed score estimator in reverse order.

\paragraph{Score estimator.}

We begin by presenting the score estimator assuming access to subspace estimates $\{A_i\}_{i=1}^M$
and a classification function $c(\cdot)$.

Inspired by the score decomposition in (\ref{eq:score decomp}), we construct the score estimator as a weighted combination of score components associated with each subspace:
\begin{align}
    \widehat{s}_t(x) \defn  \sum_{i=1}^M \widehat{w}_t(i,x) \widehat{s}_t(i,x) \label{eq:score estimator},
\end{align}
where $\widehat{w}_t(i,x)$ and $\widehat{s}_t(i,x)$ estimate $w_t(i,x)$ in \eqref{eq:defn w_t(i,x)} and $s_t(i,x)$ in \eqref{eq:defn s_t(i,x)}, respectively. 
\begin{itemize}

\item \textit{Score component estimator $\widehat{s}_t(i,x)$.}
In light of the low-dimensional structure in (\ref{eq:s_t(i,x) decomp}), it suffices to learn an estimator $\widehat{s}_t^\low$ for each $k_i$-dimensional score $s_t^\low(i,\cdot)$, and then construct the estimator for the $d$-dimensional score $s_t(i,x)$ associated with $V_i$ as
\begin{align}
    \widehat{s}_t(i,x) \defn -\frac{x-\proj_i(x)}{t} + A_i \widehat{s}_t^\low(i,A_i^{\top}x).
    \label{eq:s_hat_t(i,x) decomp}
\end{align}

To build $\widehat{s}_t^\low$, let $\cC_i \defn \{j\in [N]: X^{(j)} \in V_i \}$ denote the index set of samples belonging to subspace $V_i$. 
Since $s_t^\low(i,\cdot)$ is the score  of the smoothed low-dimensional distribution $p_i^\low\ast\mathcal{N}(0,tI_{k_i})$,   
we first estimate its density using the Gaussian kernel estimator
\begin{align}
    \widehat{g}_t(i,x) \defn \frac{1}{|\cC_i|} \sum_{j\in \cC_i} \varphi_{t}(x - A_i^{\top} X^{(j)};k_i) \label{eq:estimator for p_t^l(i,x)},
\end{align}
where $\varphi_{t}(x;k_i)$ denotes the density of $\mathcal{N}(0,tI_{k_i})$. 
We then define the $k_i$-dimensional score estimator as
\begin{align}
    \widehat{s}_t^\low(i,x) \defn \clip_{R} \bigg( \frac{\nabla \widehat{g}_t(i,x)}{\widehat{g}_t(i,x)} \psi \Big(\widehat{g}_t(i,x);\frac{\log N}{N(2\pi t)^{k_i/2}}\Big)\bigg). \label{eq:low-dimensional score estimator}
\end{align}
Here, $\psi(x;\eta) \defn \ind\{x \geq \eta\}$ is a thresholding function, and the clip operator is defined by
\begin{align*}
\clip_r(z)\defn  
\begin{cases}
    z, &\|z\|_2 \leq  r; \\
    \frac{z}{\|z\|}\cdot r,&\text{otherwise}.
\end{cases}
\end{align*}
We set the clipping radius to be $R = \sqrt{2\log N/t}$.

In words, we first form the plug-in estimator $\nabla \widehat{g}_t/\widehat{g}_t$ using the kernel density estimator~\eqref{eq:estimator for p_t^l(i,x)}. 
We then apply a thresholding rule $\psi(\widehat{g}_t;\eta_t)$, which regularizes this ratio according to the estimated density level $\widehat{g}_t$ and the threshold $\eta_t = N^{-1}(2\pi t)^{-k_i/2}\log N$ that depends on the sample size $N$ and time $t$.
Specifically, in low-density regions where $\nabla \widehat{g}_t/\widehat{g}_t$ is unstable due to small denominators and limited data, the resulting score estimator $\widehat{s}_t$ is set to zero.
This regularization is important not only for controlling the subsequent estimation error, but also for improving generalization by preventing the estimator from closely fitting empirical artifacts.

\item \textit{Weight estimator $\widehat{w}_t(i,x)$.} 
As for the mixture weight $w_t(i,x)$ defined in (\ref{eq:defn w_t(i,x)}), we first construct Gaussian kernel density estimators for $p_t(x)$ and $q_t(i,x)$:
        \begin{subequations}
            \begin{align}
                \widehat{p}_t(x) \defn \frac{1}{N} \sum_{j=1}^N \varphi_t(x-X^{(j)};d), \label{eq:estimator for p_t(x)}
            \end{align}
            and for any $i\in [M]$,
            \begin{align}
                \widehat{q}_t(i,x) \defn \frac{1}{N} \sum_{j=1}^N \varphi_t(x-X^{(j)};d) \ind_{\{ c(X^{(j)})=i\}}. \label{eq:estimator for p_t(i,x)}
            \end{align}
        \end{subequations}
        We then define the weight estimator as
\begin{align}
    \widehat{w}_t(i,x):= \frac{\widehat{q}_t(i,x)}{\widehat{p}_t(x)} \ind_{\{ x \in \mathcal{G}_t(i)\}}, \label{eq:estimate w_t(i,x)}
\end{align}
where $\mathcal{G}_t(i)$ is a set given by
\begin{align}
    \mathcal{G}_t(i) \defn \Big\{ x: \| x-\proj_i(x)\|_2\leq R_t(i)\Big\},  \label{eq:regularization set}
\end{align}
with $R_t(i) = C_R\sqrt{ td \log (Ndt^{k_i/2})}$ for some universal constant $C_R >0$. 

In a word, the weight estimator $\widehat{w}_t(i,x)$ is a plug-in estimator for the true weight \eqref{eq:defn w_t(i,x)}, up to some low-probability set under $p_t$.
The indicator function $\ind_{\{ x \in \mathcal{G}_t(i)\}}$ is introduced for technical convenience in the analysis and could be removed with a shaper argument.

\end{itemize}

\paragraph{Subspace recovery.}
Finally, we briefly discuss how the subspace estimates $\{A_i\}_{i=1}^M$ and classification function $c(\cdot)$ can be obtained from training data.
This is a classical subspace clustering problem. Under standard identifiability and separation conditions, and assuming known bounds on the number of subspaces $M$ and the maximal intrinsic dimension $k$, several polynomial-time methods can recover the underlying subspaces and cluster assignments, such as sparse subspace clustering \cite{elhamifar2013sparsesubspaceclusteringalgorithm}, thresholding-based subspace clustering \cite{heckel2015robustsubspaceclusteringthresholding} and greedy subspace clustering \cite{park2014greedysubspaceclustering}.
From a statistical perspective, this geometric recovery step is typically less demanding than learning the full target distribution.

\subsection{Theoretical guarantees}

We now state our theoretical guarantees for the proposed score estimator and the resulting sampler.

We first present the $L^2$ error for the proposed score estimator $\wh s_t$ in (\ref{eq:score estimator}).
The proof can be found in Appendix \ref{pf-of-thm:thm:score estimation error}.
\begin{theorem}
    \label{thm:score estimation error}
    Suppose the target distribution $p^\star$ satisfies Assumptions~\ref{assume:multi-modal} and \ref{assump:sub-gaussian target}.
    Under the event of exact subspace recovery and $t\leq N^{O(1)}$, the $L^2$-error of the score estimator in (\ref{eq:score estimator}) using $N$ samples satisfies
    \begin{align*}
        &\mathbb{E}\big[\|\widehat{s}_t(X)-s_t^\star(X)\|_2^2\big] \\
        & \quad \leq  C_\score \frac{dM^3}{N} \Big( \frac{1}{t}+\frac{\sigma^{k\vee2}}{t^{{(k\vee2)}/2+1}}\Big)  \polylog N 
    \end{align*}
    for some constant $C_\score >0$ independent of $N$, $d$, $M$ and $t$.
    The expectation here is taken over the i.i.d.\ training samples $\{ X^{(i)}\}_{i=1}^N$ used for score estimation and $X \sim p_t$.
\end{theorem}

In words, Theorem \ref{thm:score estimation error} shows that the convergence rate of the $L^2$ error (with respect to diffusion time $t$) of the proposed score estimator depends on the intrinsic dimension $k$, rather than the ambient dimension~$d$. 
This yields a substantial improvement over existing rate-optimal score estimation guarantees for general distributions  \citep{wibisono2024optimal,zhang2024minimaxoptimalityscorebaseddiffusion,dou2024optimalscorematchingoptimal,cai2025minimaxoptimalityprobabilityflow},which do not exploit intrinsic low-dimensional data structures and therefore suffer from the curse of dimensionality.

Moreover, we note that exact subspace recovery can be achieved with high probability using $n_0=C_{\mathsf{sc}}M^2k\log n$ samples for subspace clustering, for a sufficiently large constant $C_{\mathsf{sc}}>0$. This sample size is negligible compared with the remaining $N=n-n_0$ samples used for score estimation, provided that $n$ is sufficiently large.

We next translate the resulting score estimation guarantee into a sampling guarantee for the diffusion sampler. 
The proof is deferred to Appendix~\ref{pf-of-thm:thm:high prob TV}.

\begin{theorem}
    \label{thm:high prob TV}
    Suppose the target distribution $p^\star$ satisfies Assumptions~\ref{assume:multi-modal} and \ref{assump:sub-gaussian target}.
    Let $n_0 = C_{\mathsf{sc}}M^2 k \log n$ for some large constant $C_{\mathsf{sc}} > 0$ and $N=n-n_0$. Then for sufficiently large $n$, the output $\widehat{Y}_{T-\tau}$ of Algorithm \ref{alg:reverse SDE sampling}, using the score estimator in (\ref{eq:score estimator}) constructed from $N$ samples with  $T=\log n$ and $\tau = n^{-2/k}$, satisfies
    \begin{align}\label{eq:W_1 bound}
        \mathbb{E}\big[W_1(p^\star,p_{\widehat{Y}_{T-\tau}})\big] \leq C d M^{3/2} n^{-\frac1{k\vee 2}} \polylog n
    \end{align}
    for some constant $C>0$ independent of $n$, $d$ and $M$.
    Here the expectation is taken over the samples $\{X^{(i)}\}_{i=1}^n$.
\end{theorem}

Theorem~\ref{thm:high prob TV} provides the convergence rate of the $W_1$-sampling error for diffusion sampling equipped with the proposed score estimator. 
By exploiting the intrinsic low-dimensional data structure through kernel-based score estimation, the resulting convergence rate (with respect to sample size $n$) is governed by the intrinsic dimension $k$, rather than the ambient dimension $d$, with $d$ appearing only linearly through the prefactor.

Several remarks are in order:
\begin{itemize}
    \item \textit{(Near-)minimax optimality.} The minimax risk of estimating a $k$-dimensional density scales as $n^{-\frac{1}{k \vee 2}}$ \citep{chewi2024statisticaloptimaltransport}.
    Since sampling is always harder than density estimation, Theorem \ref{thm:high prob TV} shows that our sampling algorithm is minimax optimal (up to logarithmic factors).
    
    \item \textit{Sample complexity.} The error bound in \eqref{eq:W_1 bound} demonstrates that in order to achieve an $\veps$-accurate sampling in 1-Wasserstein distance, it suffices to have $\veps^{- (k \vee 2)}$ samples, up to logarithmic factors, thereby breaking the curse of dimensionality that plagues prior results \citep{wibisono2024optimal,zhang2024minimaxoptimalityscorebaseddiffusion,dou2024optimalscorematchingoptimal,cai2025minimaxoptimalityprobabilityflow}.
    \item \textit{Weak assumptions on the target distribution.} 
    Our results do not rely on stringent structural conditions commonly imposed in earlier work, such as smooth densities/scores, log-concavity, or exactly Gaussian components.
    As a consequence, our framework applies to a broader class of multi-modal distributions of practical interest. 
    Moreover, to the best of our knowledge, even in the single low-dimensional setting, our result is the first to achieve a (near-)optimal rate under only a subgaussian assumption on the target distribution, without extra assumptions on the score or density.
\end{itemize}
\begin{remark}
We believe that the linear dependence on $d$ in the prefactor of \eqref{eq:W_1 bound} is likely a proof artifact and may be improved through a sharper analysis. Determining whether this dependence is intrinsic or can be removed is an interesting direction for future work.
\end{remark}

\begin{remark}
Prior theory \citep{cai2025minimaxoptimalityprobabilityflow} suggests that, once the score estimation error is controlled, the discretization error of practical diffusion samplers does not affect the final statistical rate of the sampling error. Accordingly, this work focuses on the main statistical bottleneck, namely score estimation, by analyzing the idealized continuous-time reverse process.
Meanwhile, establishing sharp Wasserstein discretization bounds under mild distributional conditions remains an important direction for future work.
\end{remark}

%% file: Proof_Overview.tex
\section{Analysis}
\label{sec:proof overview}
This section provides the proof sketches for Theorems \ref{thm:score estimation error}--\ref{thm:high prob TV}.

\paragraph{Proof sketch of Theorem \ref{thm:score estimation error}.}

In light of the expressions of the true score (\ref{eq:score decomp}) and the score estimator (\ref{eq:score estimator}), the $L^2$-error decomposes as:
\begin{align}
    &\mathbb{E}\big[\|\widehat{s}_t(Z_t)-s_t^\star(Z_t)\|_2^2\big] \notag \\
    &\,\,\, \lesssim \sum_{i=1}^M \int \bE\big[\big(w_t(i,x)-\widehat{w}_t(i,x)\big)^2 \|\widehat{s}_t(i,x) \|_2^2 \big] p_t(x) \!\diff x \notag\\
    &\quad + \sum_{i=1}^M \int w_t^2(i,x) \, \bE\big[\| s_t(i,x) - \widehat{s}_t(i,x) \|_2^2 \big] p_t(x) \! \diff x, \label{eq:22}
\end{align}
where the expectation in the first line is taken over both $Z_t\sim p_t$ and the i.i.d.\ samples $\{ X^{(i)}\}_{i=1}^N$, while the second and the third line only take expectation over the samples.

This decomposition suggests that we need to control the  mean squared error of both the weight estimator $\widehat{w}_t(i,x)$ and score estimator $\widehat{s}_t(i,x)$.

For the weight estimator, we recall the true weight $w_{t}(i,x) = q_t(i,x)/p_t(x)$ in (\ref{eq:defn w_t(i,x)}).
The mean squared error is controlled in the following lemma, with proof deferred to Appendix \ref{pf-of-lem:lem:MSE for weight}.
\begin{lemma}
    \label{lem:MSE for weight}
    For any $x \in \mathcal{G}_t(i)$, the weight estimator in (\ref{eq:estimate w_t(i,x)}) satisfies
    \begin{align*}
        &\mathbb{E}\big[\big(w_t(i,x) -\widehat{w}_t(i,x) \big)^2 \big] \\
        & \quad\lesssim  \frac{1}{p_t^2(x)} \frac{1}{t^{d/2}N} \Big(\sum_{i=1}^M e^{-\frac{1}{2t}\|x-\proj_i(x)\|_2^2} q_t(i,x) \Big),
    \end{align*}
    where the expectation is taken over the i.i.d.\ samples $\{ X^{(i)}\}_{i=1}^N$. 
\end{lemma}

\begin{remark}
    The above mean squared error bound depends on the point $x$, density $p_t(x)$ and the geometric structure. This enables us to obtain a tight $L^2$ error bound and further get rid of the dependence on ambient dimension $d$. 
\end{remark}

With Lemma \ref{lem:MSE for weight} in hand, we can now bound the first term in \eqref{eq:22} associated with weight estimation.
It is easy to see $p_t$ is $\sqrt{\sigma^2+t}$-subgaussian. Define
$$\mathcal{B}_t \defn \big\{ x \in \mathbb{R}^d: \| A_i^{\top} x\|_{2} \leq B_t, \forall i \in [M] \big\},$$
with $B_t \defn C_B \sqrt{k(\sigma^2+t)\log N}$ for some universal constant $C_B>0$. One can show that the estimation error within $\mathcal{B}_t$ dominates since $\mathcal{B}_t^c$ is a low probability region w.r.t. $p_t$.
This allows us to apply Lemma \ref{lem:MSE for weight} to
derive the following bound that only depends on $k_i$: 
\begin{align}
    & \int \mathbb{E} \big[  (w_t(i,x)-\widehat{w}_t(i,x))^2\big]p_t(x)\diff x \notag \\
    &\quad  \lesssim 
    \sum_{j=1}^M\int_{\mathcal{B}_t} \frac{1}{N t^{d/2}} e^{-\frac{1}{2t} \| x-A_jA_j^{\top}x\|_2^2}  \diff x \notag  \\
    & \quad = \wt O\bigg( 
    \sum_{j=1}^M \frac{(\sigma^2+t)^{k_j/2}}{Nt^{k_j/2}} \bigg). \label{temp1}
\end{align}

Regarding the second term in \eqref{eq:22} associated with score estimation, notice that $w_t(i,x) = q_t(i,x)/p_t(x) \leq 1$. Thus, it suffices to bound
    \begin{align*}
        &\int w_t^2(i,x) \bE\big[\| s_t(i,x) - \widehat{s}_t(i,x) \|_2^2 \big] p_t(x) \diff x \\
        & \quad \leq \int \bE\big[\| s_t(i,x) - \widehat{s}_t(i,x) \|_2^2 \big] q_t(i,x) \diff x.
    \end{align*}
To this end, let $N_i\defn \sum_{j=1}^N \ind_{\{c(X^{(j)})=i\}}$ denote the sample size that we use to estimate the score on subspace $V_i$. The following lemma provides a mean squared error bound for the score estimator $\widehat{s}_t(i,x)$ in (\ref{eq:s_hat_t(i,x) decomp}). The proof is deferred to Appendix \ref{pf-of-lem:lem:bound of s_t(i,x)}.
\begin{lemma}
    \label{lem:bound of s_t(i,x)}
    For any fixed $n_i$, the score estimator $\widehat{s}_t(i,x)$ in (\ref{eq:s_hat_t(i,x) decomp}) satisfies
    \begin{align}
        &\int_{\mathbb{R}^d} \mathbb{E}\big[\| \widehat{s}_t(i,x) - s_t(i,x)\|_2^2 \ind_{\{N_i \geq n_i \}}\big] q_t(i,x) \diff x \notag \\
        &\quad \leq C_{k_i} \frac{p_i^\star(V_i)}{n_i}\Big( \frac{1}{t}+\frac{\sigma^{k_i}}{t^{k_i/2+1}}\Big)\polylog N  \label{eq:bdd of s_t(i,x)}
    \end{align}
    for some constant $C_{k_i}>0$ only depending on $k_i$.
    In addition, for any $x$, the $\ell_2$-norm of $\widehat{s}_t(i,x)$ is bounded by
    \begin{align}\label{eq:bdd of norm of s_t(i,x)}
        \| \widehat{s}_t(i,x)\|_2 \lesssim \frac{\|x-\proj_i(x)\|_2}{t} + \sqrt{\frac{2}{t} \log N}.
    \end{align}
\end{lemma}

Combining \eqref{temp1}--\eqref{eq:bdd of norm of s_t(i,x)} completes the proof of Theorem \ref{thm:score estimation error}.

\paragraph{Proof sketch of Theorem~\ref{thm:high prob TV}.}
As we will show in Lemma \ref{lem:subspace clustering} in Appendix~\ref{sec:pf-of-lemmas}, $n_0=C_{\mathsf{sc}}M^2k\log n$ samples suffice to recover the subspaces exactly with probability at least $1-Mn^{-10}$, for a sufficiently large constant $C_{\mathsf{sc}}>0$. Thus, for large $n$, the remaining sample size for score estimation satisfies $N=n-n_0 \geq n/2$.
Conditioned on the exact subspace recovery event, we apply the score estimation error bound in Theorem \ref{thm:score estimation error}. We then relate the 1-Wasserstein error between the target distribution $p^\star$ and the generated distribution $p_{\widehat{Y}_{T-\tau}}$ 
to the integral of the score estimation error over time via the following stability bound \citep{oko2023diffusionmodelsminimaxoptimal,azangulov2025convergencediffusionmodelsmanifold,tang24a}:
\begin{align}
    & \mathbb{E}\big[W_1(p^\star,p_{\widehat{Y}_{T-\tau}})\big] \notag\\
    & \,\,\,\lesssim \sqrt{d} \, \Big( \sqrt{\tau} + \delta + e^{-T}+  \notag\\ 
    & \quad \sum_{j=0}^{L-1} \sigma_{T_{j+1}} \sqrt{\log \frac1\delta \int_{T_j}^{T_{j+1}} \mathop{\bE}\big[ \| \widehat{s}_{X_t}(X)-s_{X_t}^\star(X)\|_2^2\big] \diff t}\Big) \label{eq:W1_convergence_rate} 
\end{align}
for $0<T_0=\tau < T_1< \cdots <T_L = T$ and any $\delta>0$. Here the expectation is taken over the randomness of samples and $X \sim p_{X_t}$.

By Theorem \ref{thm:score estimation error}, together with $N\geq n/2$ and the score relation in \eqref{eq:score between VE process and VP}, one can show that
\begin{align*}
    & \int_{\mathbb{R}^d}\mathbb{E}\big[ \| \widehat{s}_{X_t}(x)-s^\star_{X_t}(x)\|_2^2\big] p_{X_t}(x) \diff x \\
    &\quad =\wt O\bigg( \frac{dM^3}{n} \Big( \frac{1}{h(t)}+\frac{\sigma^{k\vee2}}{h(t)^{{(k\vee2)}/2+1}}\Big)\bigg).
\end{align*}
Observe that 
$
    h(t) = 2e^{2t} = 2/c_t^2,
$
and thus 
\begin{align*}
    & \int_{T_j}^{T_{j+1}} \int_{\mathbb{R}^d}\mathbb{E}\big[ \| \widehat{s}_{X_t}(x)-s^\star_{X_t}(x)\|_2^2\big] p_{X_t}(x) \diff x \diff t \\
    & = \wt O\bigg( \frac{dM^3}{n} \Big( \log \frac{h(T_{j+1})}{h(T_j)} + \frac{2\sigma^{k\vee2}}{k\vee2}\frac{1}{h(T_j)^{(k\vee2)/2}} \Big)\bigg). 
\end{align*}

We then choose a dyadic partition of the time interval by setting
\begin{align*}
    T_{j+1} = 2T_j, \quad T \asymp \log n, \quad \tau\asymp n^{-\gamma} 
\end{align*}
for some $\gamma>0$ that will be specified below. Plugging these into \eqref{eq:W1_convergence_rate} yields
\begin{align*}
    &\sum_{j=0}^{L-1} \sqrt{(1-e^{-4T_j})\int_{T_j}^{T_{j+1}} \mathbb{E}\big[ \| \widehat{s}_{X_t}(X)-s^\star_{X_t}(X)\|_2^2\big] \diff t} \\
    &\lesssim \sqrt{d} \, M^{3/2} \,\polylog n 
    \begin{cases}
        \frac{1}{\sqrt{n}}\tau^{-\frac{k}{4}+\frac{1}{2}}, \quad &k\geq 2 \\
        \frac{1}{\sqrt{n}}, & k=1
    \end{cases}
\end{align*}

Finally, taking $\delta = n^{-1}$ and $\tau = n^{-2/k}$, we conclude
\begin{align*}
    \bE\big[W_1(p^\star,p_{\widehat{Y}_{T-\tau}})\big] \lesssim \frac{dM^{3/2}}{n^{1/(k\vee 2)}} \polylog n.
\end{align*}

%% file: Numerical_results.tex
\section{Numerical results}
\label{sec:numerical results}
In this section, we provide numerical experiments to validate the theoretical findings of our paper.
Since evaluating the Wasserstein distance is computationally prohibitive in high dimensions, we focus on the $L^2$-score estimation guarantee in Theorem \ref{thm:score estimation error}.

We consider a synthetic target distribution in $\mathbb{R}^d$ with $d=48$. The support is a union of $M=128$ randomly generated linear subspaces, each with intrinsic dimension $k=3$.
The restriction of the distribution to each subspace is chosen to be a two-component Gaussian mixture with randomized parameters. 
We construct the kernel-based score estimator using $N=50,000$ i.i.d samples from the target distribution. 
For each time value $t$, we approximate the $L^2$-score estimation error by Monte Carlo using $10{,}000$ independent samples from $p_t$ and average the result over $20$ independent training datasets generated from the same target distribution.

Figure~\ref{fig:numerical results} plots the empirical $L^2$ score estimation error versus the diffusion time $t$. The observed scaling is consistent with the prediction of Theorem~\ref{thm:score estimation error}, where the score estimation error is governed by the intrinsic dimension of the data, rather than the ambient dimension. In particular, despite the relatively large ambient dimension $d=48$, the empirical error exhibits a substantially milder dependence on $t$ than would be suggested by ambient-dimensional worst-case bounds.

\begin{figure}[t]
    \begin{center}
    \includegraphics[width=0.9\columnwidth]{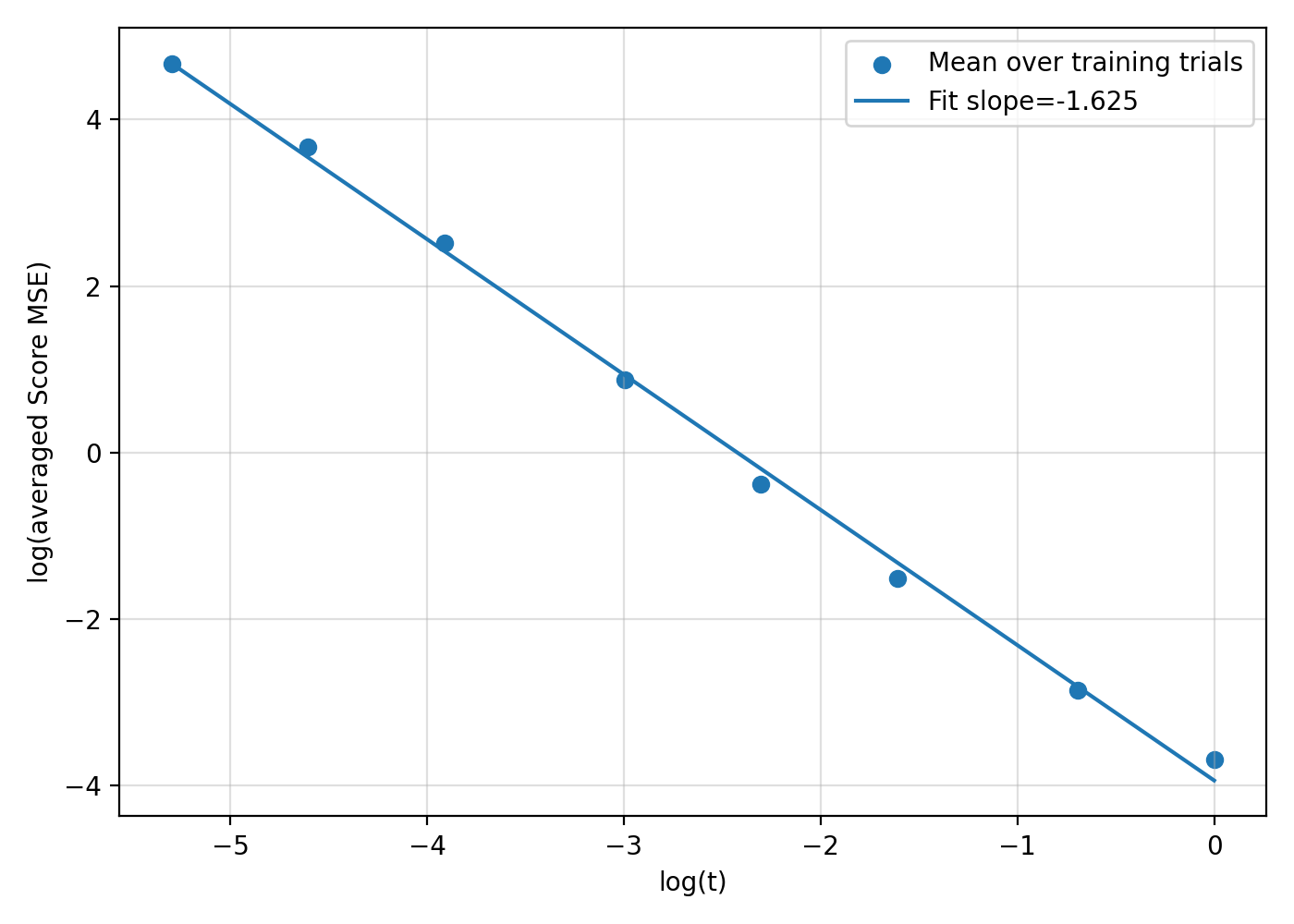}
    \caption{
       Empirical $L^2$-score error versus diffusion time $t$.\label{fig:numerical results}
    }
    \end{center}
\end{figure}

%% file: discussions.tex
\section{Discussion}
\label{sec:discussion}

This paper has studied the sample complexity of diffusion models for learning distributions supported on a union of low-dimensional subspaces, a tractable model for low-dimensional, multi-modal data commonly observed in practice. 
We construct a kernel-based score estimator and prove that diffusion-based sampling can learn the target distribution using at most $\widetilde{O}(\veps^{-(k\vee2)})$ samples, where $k$ represents the intrinsic dimension.
Our result shows that diffusion models can achieve statistical optimality by exploiting intrinsic low-dimensional structure while naturally accommodating multi-modal data.

Building on the results of our paper, several directions remain open for future work. 
First, it would be valuable to extend our theory to practical NN-based score estimators. A natural approach is to analyze an ERM estimator over a NN class, which requires controlling both approximation and generalization errors. 
The main challenge lies in the approximation step: constructing a NN approximation of the target score whose complexity depends on the intrinsic dimension. 
Our score decomposition, together with the kernel-based construction, makes the relevant low-dimensional approximation targets explicit. Thus, our construction provides a concrete roadmap for future analysis of NN score estimators.
Second, it would be important to extend our framework to real data with more complex geometric structures, such as classes residing on manifolds of varying dimensions. 
A natural first step is to consider distributions concentrated near a union of low-dimensional subspaces. In this setting, once the underlying subspaces are learned from noisy observations, the score still admits an analogous normal-tangent decomposition, where the tangent component is governed by a low-dimensional score and the normal component remains Gaussian with an enlarged variance. This suggests that the framework developed here could be extended to noisy low-dimensional models.
Third, it would be interesting to develop a fully end-to-end convergence analysis that explicitly accounts for both score estimation error and discretization of the reverse-time dynamics (SDE/ODE), ideally yielding non-asymptotic bounds in Wasserstein distance under mild distributional conditions.

%% file: pf-of-theorems.tex
\section{Proof of theorems}

\subsection{Proof of Theorem \ref{thm:score estimation error}}
\label{pf-of-thm:thm:score estimation error}
In this section, we consider the score estimation error under exact subspace recovery. 
We first define the following event set $\mathcal{A}$ as
\begin{align*}
    \mathcal{A} \defn \Big\{ N_i \geq \frac{N}{2c_{p}M}, \quad \forall i \in [M] \Big\}.
\end{align*}
Here $N_i\defn \sum_{j=1}^N \ind_{\{ c(X^{(j)}) = i\}}$ is a random number of samples on $V_i$. 
The following claim tells us that, event $\mathcal{A}$ happens with high probability.
\begin{claim}
    \label{claim:equal separation}
    Under Assumption \ref{assume:multi-modal}, it holds that,
    \begin{align*}
       \bP[\mathcal{A}^c] \leq Me^{-\frac{N}{2c_{p}^2M^2}}.
    \end{align*}
\end{claim}
The proof follows from concentration inequality and can be found in Appendix \ref{pf-of-claim:claim:equal separation}.

We first consider the error on event $\mathcal{A}$, where we have enough samples on each subspace. Notice the score decomposition (\ref{eq:score decomp}) and the score estimator in (\ref{eq:score estimator}), the $L^2$ estimation error can be written as
\begin{align*}
    \int \mathbb{E}\Big[\big\| s_t^\star(x)-\widehat{s}_t(x)\big\|_2^2 \ind_{\mathcal{A}}\Big] p_t(x) \!\diff x &=  \int \mathbb{E}[\| \sum_{i=1}^M \big(w_t(i,x)s_t(i,x)- \widehat{w}_t(i,x)\widehat{s}_t(i,x)\big)\|_2^2 \ind_{\mathcal{A}}] p_t(x) \!\diff x \\
    & \leq M \sum_{i=1}^M \int \mathbb{E}\Big[ \big\| w_t(i,x)s_t(i,x)-\widehat{w}_t(i,x)\widehat{s}_t(i,x) \big\|_2^2 \ind_{\mathcal{A}}\Big] p_t(x) \!\diff x  \quad (\text{C-S Ineq})\\
    &\defnrev M \sum_{i=1}^M L_i .
\end{align*}
We further decompose $L_i$ using the error of weight estimator and the error of score estimator respectively, 
\begin{align*}
    L_i &\defn \int \mathbb{E}\Big[ \big\| w_t(i,x)s_t(i,x)-\widehat{w}_t(i,x)\widehat{s}_t(i,x) \big\|_2^2 \ind_{\mathcal{A}}\Big] p_t(x) \!\diff x \\
    & \lesssim \int \mathbb{E} \Big[  \big(w_t(i,x)-\widehat{w}_t(i,x)\big)^2 \cdot \|\widehat{s}_t(i,x) \|_2^2 \ind_{\mathcal{A}}\Big] p_t(x) \!\diff x + \int \mathbb{E} \Big[  \big(w_t(i,x))^2 \cdot \| s_t(i,x) - \widehat{s}_t(i,x) \|_2^2 \ind_{\mathcal{A}}\Big] p_t(x) \!\diff x \\
    & \defnrev L_{i,1} + L_{i,2}.
\end{align*}

\paragraph*{Bound of $L_{i,1}$.}
We utilize a set $\mathcal{B}_t$ here and further decompose $L_{i,1}$ as
\begin{align*}
    L_{i,1} &\defn \int \mathbb{E} \Big[  \big(w_t(i,x)-\widehat{w}_t(i,x)\big)^2 \cdot \|\widehat{s}_t(i,x) \|_2^2 \ind_{\mathcal{A}}\Big] p_t(x) \! \diff x \\
    & \leq \int_{\mathcal{G}_{t}(i) \cap \mathcal{B}_t} \mathbb{E} \Big[  \big(w_t(i,x)-\widehat{w}_t(i,x)\big)^2 \cdot \|\widehat{s}_t(i,x) \|_2^2 \ind_{\mathcal{A}}\Big] p_t(x) \!\diff x  \\
    &+ \int_{\mathcal{G}_t(i)^c \cap \mathcal{B}_t} \mathbb{E} \Big[  \big(w_t(i,x)-\widehat{w}_t(i,x)\big)^2 \cdot \|\widehat{s}_t(i,x) \|_2^2 \ind_{\mathcal{A}}\Big] p_t(x) \!\diff x \\
    & + \int_{\mathcal{B}_t^c} \mathbb{E} \Big[  \big(w_t(i,x)-\widehat{w}_t(i,x)\big)^2 \cdot \|\widehat{s}_t(i,x) \|_2^2 \ind_{\mathcal{A}}\Big] p_t(x) \!\diff x \\
    &\defnrev \kappa_1 + \kappa_2 + \kappa_3,
\end{align*}
here we recall that $\mathcal{G}_t(i)$ is the regularization set defined in (\ref{eq:regularization set}) and $\mathcal{B}_t$ is defined as, 
\begin{align}
    \mathcal{B}_t \defn \Big\{ x \in \mathbb{R}^d: \| A_i^{\top} x\|_{2} \leq B_t, \forall i \in [M]\Big\}, \quad \text{with }B_t \defn C_B \sqrt{k(\sigma^2+t)\log (N)}. \label{eq:def of B_t}
\end{align}
for certain absolute constant $C_B>0$. 


The following claim tells us that the probability outside $\mathcal{B}_t$ is negligible and see Appendix \ref{pf-of-claim:claim:bdd of set B_t} for its proof. 
\begin{claim}
    \label{claim:bdd of set B_t}
    Under Assumption \ref{assump:sub-gaussian target}, for the set $\mathcal{B}_t$, it holds that, 
    \begin{align*}
        \int_{\mathcal{B}_t^c} p_t(x)\!\diff x &\lesssim \frac{Mk}{N^{4}}\\
        \int_{\mathcal{B}_t^c} \|x\|_2^2 \cdot p_t(x) \!\diff x &\lesssim \frac{d\sqrt{M}}{N^4} (\sigma^2+t).
    \end{align*}
\end{claim}

\begin{itemize}
    \item For $\kappa_1$, notice that for $x\in \mathcal{G}_t(i)$, 
    \begin{align*}
        \| \widehat{s}_t(i,x)\| _2 \leq \frac{R_t(i)}{t} + \sqrt{\frac{2}{t}\log N} 
    \end{align*}
    using Lemma \ref{lem:bound of s_t(i,x)}. Hence, 
    \begin{align*}
        \kappa_1 &\defn \int_{\mathcal{G}_{t}(i) \cap \mathcal{B}_t} \mathbb{E} \Big[  \big(w_t(i,x)-\widehat{w}_t(i,x)\big)^2 \cdot \|\widehat{s}_t(i,x) \|_2^2 \ind_{\mathcal{A}}\Big] p_t(x) \!\diff x \\
        & \lesssim \Big( \frac{R_t(i)^2}{t^2} + \frac{1}{t}\Big) \cdot \int_{\mathcal{G}_t(i) \cap \mathcal{B}_t}  \frac{1}{(2\pi t)^{d/2}N\cdot p_t(x)} \Big(\sum_{j=1}^M e^{-\frac{1}{2t}\|x-\proj_j(x)\|_2^2} \cdot q_t(j,x) \Big) \!\diff x \quad (\text{Lemma \ref{lem:MSE for weight}}) \\
        &\leq \Big( \frac{R_t(i)^2}{t^2} + \frac{1}{t}\Big) \cdot \sum_{j=1}^M \int_{\mathcal{B}_t} \frac{1}{N(2\pi t)^{d/2}} e^{-\frac{1}{2t} \| x-A_jA_j^{\top}x\|_2^2}  \!\diff x \\
        & \leq \Big( \frac{R_t(i)^2}{t^2} + \frac{1}{t}\Big) \cdot  \sum_{j=1}^M \frac{2^{k_j} B_t^{k_j}}{N(2\pi t)^{k_j/2}} \\
        &\lesssim \frac{d}{t}\cdot \frac{M}{N}\Big( 1+\frac{\sigma^k}{t^{k/2}}\Big) \cdot \big(\polylog N +\log t\big)
    \end{align*}
    In the third inequality, we apply the following changing variable technique in integration and then use Tonelli's Theorem, 
    \begin{align}
            z = T_jx \defn 
            \begin{pmatrix}
             A_j^{\top}x\\
            P_{V_j^{\perp}}^{\top} x
        \end{pmatrix} \label{eq:changing var}
        \end{align}
        here columns of $P_{V_j^{\perp}}$ denotes an orthogonal basis of $V_j^{\perp}$ and this is an orthogonal transform with
    \begin{align*}
        \|x-A_jA_j^{\top}x\|_2^2 = \|P_{V_j^{\perp}}^{\top} x\|_2^2. 
    \end{align*}

    \item For $\kappa_2$, we have, 
    \begin{align*}
        \kappa_2 &\defn \int_{\mathcal{G}_t(i)^c \cap \mathcal{B}_t} \mathbb{E} \Big[  \big(w_t(i,x)-\widehat{w}_t(i,x)\big)^2 \cdot \|\widehat{s}_t(i,x) \|_2^2 \ind_{\mathcal{A}}\Big] p_t(x) \!\diff x \\
        & = \int_{\mathcal{G}_t(i)^c \cap \mathcal{B}_t} \frac{q_t^2(i,x)}{p_t^2(x)} \mathbb{E} [\|\widehat{s}_t(i,x) \|_2^2] \cdot p_t(x)\!\diff x \\
        &\lesssim \int_{\mathcal{G}_t(i)^c \cap \mathcal{B}_t} \Big( \frac{1}{t} + \frac{\|x-A_iA_i^{\top}x\|_2^2}{t^2}\Big) \cdot q_t(i,x) \!\diff x \quad (\text{Lemma \ref{lem:bound of s_t(i,x)} + } q_t(i,x) \leq p_t(x)) 
    \end{align*}
    Since it holds that, 
    \begin{align*}
        q_t(i,x) \defn \int_{V_i} \varphi_t(x-y;d) p_i^\star(\!\diff y) \leq  p_i^\star(V_i)\cdot (2\pi t)^{-d/2} e^{-\frac{1}{2t}\|x-A_iA_i^{\top}x\|_2^2},
    \end{align*}
    we apply the similar linear transform as (\ref{eq:changing var}), 
    \begin{align*}
        \kappa_2 &\lesssim p_i^\star(V_i)\int_{\mathbb{R}^d} \ind_{ \|z\|\leq B_t} \cdot \ind_{\| z_{k_i+1:d}\|_2\geq R_t(i)} \Big( \frac{1}{t} + \frac{\| z_{k_i+1:d}\|_2^2}{t^2}\Big) \cdot (2\pi t)^{-d/2}e^{-\frac{1}{2t} \|z_{k_i+1:d}\|_2^2} \!\diff z \\
        &\lesssim p_i^\star(V_i) (2\pi t)^{-k_i/2} \int_{\mathbb{R}^{k_i}} \ind_{\| z_{1:k_i}\|\leq B_t} \Big( \frac{1}{t} + \frac{d\cdot R_t(i)^2 + d^2tC_1}{t^2}\Big) \cdot \exp \Big( -\frac{R_t(i)^2}{d\cdot tC_1}\Big)  \!\diff z_{1:k_i} \quad (\text{Lemma \ref{lem:tail bdd for subgaussian}})\\
        &\lesssim \frac{2^{k_i} B_t^{k_i}}{(2\pi t)^{k_i/2}} p_i^\star(V_i)\cdot \Big( \frac{d^2}{t} + \frac{d\cdot R_t(i)^2}{t^2}\Big)  \exp \Big( -\frac{R_t(i)^2}{d\cdot tC_1}\Big)  \\
        &\lesssim \frac{d}{Nt}\Big( 1+\frac{\sigma^{k_i}}{t^{k_i/2}}\Big)\cdot \polylog N \lesssim \frac{d}{Nt}\Big( 1+\frac{\sigma^{k}}{t^{k/2}}\Big)\cdot \polylog N. 
    \end{align*}
    Here $C_1>0$ is a universal constant which is related with the sub-gaussian norm of standard Gaussian distribution. 

    \item For $\kappa_3$, 
    \begin{align*}
        \kappa_3 &\defn \int_{\mathcal{B}_t^c} \mathbb{E} \Big[  \big(w_t(i,x)-\widehat{w}_t(i,x)\big)^2 \cdot \|\widehat{s}_t(i,x) \|_2^2 \ind_{\mathcal{A}}\Big] p_t(x) \!\diff x \\
        &\leq \int_{\mathcal{B}_t^c} \mathbb{E} \Big[   \|\widehat{s}_t(i,x) \|_2^2\Big] p_t(x) \!\diff x \\
        & \lesssim \int_{\mathcal{B}_t^c} \Big( \frac{\| x-A_iA_i^{\top}x\|_2^2}{t^2} + \frac{1}{t}\Big) \cdot p_t(x)\!\diff x \quad (\text{Lemma \ref{lem:bound of s_t(i,x)}}) \\
        &\leq \int_{\mathcal{B}_t^c} \Big( \frac{\|x\|_2^2}{t^2} + \frac{1}{t} \Big) p_t(x)\!\diff x \\
        &\lesssim \frac{dM(\sigma^2+t)}{t^2N^2} \polylog N  \quad (\text{Claim \ref{claim:bdd of set B_t}}) \\
        &\lesssim \frac{d}{tN} \Big( 1+\frac{\sigma^{k\vee2}}{t^{{(k\vee2)}/2}}\Big)\polylog N .
    \end{align*}

     \item For $L_{i,1}$, we could sum them up, 
    \begin{align*}
        L_{i,1} &\leq \kappa_1 + \kappa_2 + \kappa_3\\ 
        &\lesssim \frac{dM}{Nt} \Big( 1+\frac{\sigma^{k\vee2}}{t^{{(k\vee2)}/2}}\Big) \cdot \big(\polylog N+\log t\big) .
    \end{align*}
\end{itemize}

\paragraph*{Bound of $L_{i,2}$.}
\begin{align*}
    L_{i,2} &\defn \int \mathbb{E} \Big[  \big(w_t(i,x))^2 \cdot \| s_t(i,x) - \widehat{s}_t(i,x) \|_2^2 \cdot \ind_{\mathcal{A}}\Big] p_t(x) \!\diff x \\
    &= \int \frac{q_t^2(i,x)}{p_t^2(x)}  \mathbb{E} \Big[ \big\| s_t(i,x) - \widehat{s}_t(i,x) \big\|_2^2\cdot \ind_{\mathcal{A}} \Big] p_t(x) \!\diff x \\
    & \leq \int  \bE \Big[ \big\| s_t(i,x) - \widehat{s}_t(i,x)\big \|_2^2 \cdot \ind_{\{ N_i\geq N/2c_{p}M\}} \Big] q_t(i,x) \! \diff x \quad (\text{since }q_t(i,x) \leq p_t(x)) \\
    &\lesssim p^\star(V_i) \frac{c_{p}M(4/\sqrt{\pi})^{k_i} }{N} \Big( \frac{1}{t} + \frac{\sigma^{k_i}}{t^{k_i/2+1}}\Big) \polylog N  \quad (\text{Lemma \ref{lem:bound of s_t(i,x)}}).
\end{align*}

\paragraph*{Bound of error on $\mathcal{A}$.}
Therefore, 
\begin{align*}
     \int \mathbb{E}[\| s_t^\star(x)-\widehat{s}_t(x)\|_2^2\ind_{\mathcal{A}}] p_t(x)\!\diff x &\leq M\sum_{i=1}^M (L_{i,1}+L_{i,2}) \\
     &\lesssim \frac{dM^3}{Nt} \Big( 1+\frac{\sigma^{k\vee2}}{t^{{(k\vee2)}/2}}\Big) \cdot \big(\polylog N +\log t\big).
\end{align*}
Here for simplicity, we omit constant terms $c_{p}$ and terms that are only related with intrinsic dimension $k$.

\paragraph*{Bound of error outside $\mathcal{A}^c$.}
\begin{align*}
    \int \mathbb{E}[\| s_t^\star(x)-\widehat{s}_t(x)\|_2^2\ind_{\mathcal{A}^c}] p_t(x)\!\diff x &\lesssim  \Big(\int \|s_t(x)\|_2^2p_t(x) \!\diff x\Big) \cdot \mathbb{P}[\mathcal{A}^c] \\
    &+ \int \Big( \frac{1}{t} + \frac{\|x\|_2^2}{t^2} \Big) p_t(x)\! \diff x \cdot \mathbb{P}[\mathcal{A}^c] \quad (\text{Lemma \ref{lem:bound of s_t(i,x)}})\\
    &\lesssim \Big(\frac{d}{t} + \frac{d(\sigma^2+t)}{t^2} \Big) \cdot 2M e^{-\frac{N}{2c_{p}^2M^2}} 
\end{align*}
Here we apply Lemma 11 in \cite{cai2025minimaxoptimalityprobabilityflow} and that $p^\star \ast \mathcal{N}(0,tI_d)$ is $c\sqrt{\sigma^2+t}$ subgaussian r.v as we have proven in Appendix \ref{pf-of-claim:claim:bdd of set B_t}. 
Notice that, for large enough $N$, $e^{-\frac{N}{2c_{p}^2M^2}} \leq \frac{M^2c_{p}^2}{N}$ and thus, 
\begin{align*}
    \int \mathbb{E}[\| s_t^\star(x)-\widehat{s}_t(x)\|_2^2 \ind_{\mathcal{A}^c}] p_t(x)\!\diff x \lesssim \frac{dM^3}{Nt}\Big( 1+\frac{\sigma^2}{t}\Big).
\end{align*}

\paragraph*{Summary.}
In summary, our analysis above shows that, 
\begin{align*}
    \int \mathbb{E}\big[\|\widehat{s}_t(x)-s_t^\star(x)\|_2^2\big] p_t(x) \!\diff x &= \int \mathbb{E}[\| s_t^\star(x)-\widehat{s}_t(x)\|_2^2 \ind_{\mathcal{A}}] p_t(x) \!\diff x + \int \mathbb{E}[\| s_t^\star(x)-\widehat{s}_t(x)\|_2^2 \ind_{\mathcal{A}^c}] p_t(x) \!\diff x \\
    &\lesssim \frac{dM^3}{Nt}\Big( 1+\frac{\sigma^{k\vee2}}{t^{{(k\vee2)}/2}}\Big)  \cdot \big(\polylog N+\log t\big) .
\end{align*}
This proves Theorem \ref{thm:score estimation error}.

\subsection{Proof of Theorem \ref{thm:high prob TV}}
\label{pf-of-thm:thm:high prob TV}
\paragraph*{Proof of (\ref{eq:score between VE process and VP}).}
Recall the forward process (\ref{eq:forward conditional dist}), it holds that, 
\begin{align*}
    p_{X_t}(x) &= \int p_{X_t|X_0}(x|y) p^\star(\! \diff y) \\
    &= \int \big(2\pi\sigma_t^2\big) ^{-d/2} e^{-\frac{1}{2\sigma_t^2}\|x-c_ty\|_2^2}p^\star( \! \diff y)
\end{align*}
Hence, 
\begin{align*}
    \nabla_x p_{X_t}(x)&= (2\pi\sigma_t^2\big) ^{-d/2}\cdot  \int \Big( -\frac{x-c_ty}{\sigma_t^2}\Big)e^{-\frac{1}{2\sigma_t^2}\|x-c_ty\|_2^2}p^\star(\! \diff y)
\end{align*}
Therefore, its score function, 
\begin{align*}
    s^\star_{X_t}(x)&= \frac{\nabla_xp_{X_t}(x)}{p_{X_t}(x)} \\
    &= \frac{\int \Big( -\frac{x-c_ty}{\sigma_t^2}\Big)e^{-\frac{1}{2\sigma_t^2}\|x-c_ty\|_2^2}p^\star(\! \diff y) }{\int  e^{-\frac{1}{2\sigma_t^2}\|x-c_ty\|_2^2}p^\star(\! \diff y)}
\end{align*}
Similarly, for any $t>0$ and the VE process (\ref{eq:VE process}), it holds that, 
\begin{align*}
    s^\star_{Z_t}(x) = \frac{\int \Big( -\frac{x-y}{t}\Big)e^{-\frac{1}{2t}\|x-y\|_2^2}p^\star(\! \diff y) }{\int  e^{-\frac{1}{2t}\|x-y\|_2^2}p^\star(\! \diff y)}
\end{align*}

Therefore, 
\begin{align*}
    s^\star_{X_t}(c_tx) &=  \frac{\int \Big( -\frac{c_t(x-y)}{\sigma_t^2}\Big)e^{-\frac{c_t^2}{2\sigma_t^2}\|x-y\|_2^2}p^\star(\! \diff y) }{\int  e^{-\frac{c_t^2}{2\sigma_t^2}\|x-y\|_2^2}p^\star(\! \diff y)}\\
    &= \frac{1}{c_t}s^\star_{Z_{h(t)}}(x) \\
    &=: \frac{1}{c_t}s^\star_{h(t)}(x)
\end{align*}
Here $h(t)\defn \frac{\sigma_t^2}{c_t^2} = \frac{1-e^{-2t}}{e^{-2t}} = e^{2t}-1$. 

\paragraph*{Proof of Theorem \ref{thm:high prob TV}.}
We denote $\mathcal{E}$ as the event set of exact subspace recovery using $n_0$ samples. From Lemma \ref{lem:subspace clustering}, we could take $n_0 = O(M^2k\log n)$ and thus $\bP [\mathcal{E}^c] \lesssim Mn^{-10}$. 
Hence, for large $n$ such that $n_0 \leq 0.5 n$, we have the sample size for score estimation: $N \geq 0.5n$.
Therefore, for the score estimation error of VE process, we have, 
\begin{align}
    \bE \big[ \| \widehat{s}_{t}(X)-s_{t}^\star(X)\|_2^2\big] &= \bE \big[ \| \widehat{s}_{t}(X)-s_{t}^\star(X)\|_2^2 \cdot \ind_{\mathcal{E}} \big]  + \bE \big[ \| \widehat{s}_{t}(X)-s_{t}^\star(X)\|_2^2 \cdot \ind_{\mathcal{E}^c} \big] \notag \\
    & \lesssim \frac{dM^3}{nt} \Big( 1+\frac{\sigma^{k\vee2}}{t^{{(k\vee2)}/2}}\Big) \cdot \big(\polylog n +\log t\big) +\frac{d\log n}{t} \cdot Mn^{-10} + \sqrt{\bE[\| s_t^\star(X)\|_2^4] \cdot Mn^{-10}} \notag\\
    &\lesssim \frac{dM^3}{nt} \Big( 1+\frac{\sigma^{k\vee2}}{t^{{(k\vee2)}/2}}\Big) \cdot \big(\polylog n +\log t\big) \label{eq:score estimation error for VE}
\end{align} 
In the first inequality, we apply Theorem \ref{thm:score estimation error} and that $\|\widehat{s}_t(x)\|^2 \lesssim \frac{R_t^2(i)}{t^2} + \frac{\log n}{t} \lesssim \frac{d\log n}{t}$; in the last inequality, we apply Lemma 11 in \cite{cai2025minimaxoptimalityprobabilityflow} for the moment bound of true scores.

Now we consider the sampling error results from score estimation error. Notice that 1-Wasserstein distance between target distribution and generated distribution $\widehat{Y}_{T-\tau}$ using Algorithm \ref{alg:reverse SDE sampling} has the following control, 
\begin{align}
    \mathbb{E}\big[W_1(p^\star,p_{\widehat{Y}_{T-\tau}})\big]\lesssim \sqrt{d} \Big( \sqrt{\tau} + \sum_{j=0}^{L-1} \sqrt{\log \delta^{-1}\cdot \sigma_{T_{j+1}}^2 \int_{T_j}^{T_{j+1}} \int_{\mathbb{R}^d}\mathbb{E}\big[ \| \widehat{s}_{X_t}(x)-s_{X_t}^\star(x)\|_2^2\big] p_{X_t}(x) \diff x \diff t} + \delta + e^{-T} \Big) \label{eq:W_1 convergence rate} 
\end{align}
for certain $0<T_0=\tau < T_1< \cdots <T_L = T$ and any $\delta>0$. 
This bound can be found as (8) in \citet{azangulov2025convergencediffusionmodelsmanifold}, which cites Lemma D.7 in \citet{oko2023diffusionmodelsminimaxoptimal} as the proof, and also can be found as Lemma B.2 in \citet{tang24a}. Furthermore, the score estimator we define in (\ref{eq:score estimator}) satisfies, 
\begin{align*}
    \widehat{w}_t(i,x) \neq 0 \Longrightarrow x \in \mathcal{G}_t(i) \Longrightarrow \|\widehat{s}_t(i,x)\| \lesssim \sqrt{\frac{\log n}{t}} 
\end{align*}
and hence, $\|\widehat{s}_t(x)\|\lesssim \sqrt{\frac{\log n}{t}}$, with, 
\begin{align*}
    \widehat{s}_{X_t}(x) = \frac{1}{c_t } \widehat{s}_{h(t)}(\frac{x}{c_t}) \lesssim \frac{1}{c_t} \sqrt{\frac{\log n}{\sigma_t^2/c_t^2}} = \sqrt{\frac{\log n}{\sigma_t^2}}.
\end{align*}
Therefore, the condition for (\ref{eq:W_1 convergence rate}) is satisfied in our setting. We first derive a bound for $ \int_{\mathbb{R}^d}\mathbb{E}\big[ \| \widehat{s}_{X_t}(x)-s^\star_{X_t}(x)\|_2^2\big] p_{X_t}(x)\diff x $ using the result in Theorem \ref{thm:score estimation error}, 

\begin{align*}
    \int_{\mathbb{R}^d}\mathbb{E}\big[ \| \widehat{s}_{X_t}(x)-s^\star_{X_t}(x)\|_2^2\big] p_{X_t}(x)\diff x &= \frac{1}{c_t^2}\int_{\mathbb{R}^d}\mathbb{E}\big[ \| \widehat{s}_{h(t)}(\frac{x}{c_t})-s^\star_{h(t)}(\frac{x}{c_t})\|_2^2\big] p_{X_t}(x)\diff x \quad (\text{Equation (\ref{eq:score between VE process and VP})})\\ 
    &= \int \frac{1}{c_t^2}(c_t^2)^{d/2}\mathbb{E} \big[\| \widehat{s}_{h(t)}(y)-s^\star_{h(t)}(y)\|_2^2\big] p_{X_t}(c_t y)\diff y \\
    &= \frac{1}{c_t^2} \int \mathbb{E} \big[\| \widehat{s}_{h(t)}(y)-s^\star_{h(t)}(y)\|_2^2\big] p_{Z_{h(t)}}(y)\diff y \\
    &\lesssim \frac{1}{c_t^2} \frac{dM^3}{n} \Big( \frac{1}{h(t)}+\frac{\sigma^{k\vee2}}{h(t)^{{(k\vee2)}/2+1}}\Big) \big(\polylog n+\log h(t)\big)\quad (\text{Equation (\ref{eq:score estimation error for VE})}).
\end{align*}
Observe that, 
\begin{align*}
    h'(t) = 2e^{2t} = 2/c_t^2
\end{align*}
and thus, 
\begin{align*}
    \int_{T_j}^{T_{j+1}} \int_{\mathbb{R}^d}\mathbb{E}\big[ \| \widehat{s}_{X_t}(x)-s^\star_{X_t}(x)\|_2^2\big] p_{X_t}(x)\diff x \diff t &\lesssim \frac{dM^3}{n}\Big(\int_{T_j}^{T_{j+1}} \Big( \frac{1}{h(t)}+\frac{\sigma^{k\vee2}}{h(t)^{(k\vee2)/2+1}} \Big) h'(t) \diff t\Big) \big(\polylog n+\log h(t)\big)\\
    &\lesssim \frac{dM^3}{n} \Big( \log \frac{h(T_{j+1})}{h(T_j)} + \frac{2\sigma^{k\vee2}}{k\vee2}\frac{1}{h(T_j)^{(k\vee2)/2}} \Big) \big(\polylog n+T\big) .
\end{align*}

We then take the synthesized discretization as, 
\begin{align*}
    T_{j+1} = 2T_j, \quad T \asymp \log n, \quad \tau\asymp n^{-\gamma} 
\end{align*}
for some $\gamma>0$ that will be determined later. Then we can easily check that $L\asymp \log n$ and,
\begin{align*}
    &\sum_{j=0}^{L-1} \sqrt{(1-e^{-4T_j})\cdot \int_{T_j}^{T_{j+1}} \int_{\mathbb{R}^d}\mathbb{E}\big[ \| \widehat{s}_{X_t}(x)-s^\star_{X_t}(x)\|_2^2\big] p_{X_t}(x) \diff x \diff t} \\
    &\lesssim \frac{\sqrt{d}M^{3/2} \polylog n}{\sqrt{n}}  \Big( \sum_{j=0}^{L-1} \sqrt{(1-e^{-4T_j})\frac{2\sigma^{k\vee2}}{k\vee2}\frac{1}{(e^{2T_k}-1)^{(k\vee2)/2}}} + \polylog n\Big)\\
    &\lesssim \frac{\sqrt{d}M^{3/2}\polylog n}{\sqrt{n}} \cdot \Big(\sum_{j=0}^{L-1}\sqrt{\frac{4\sigma^kT_j}{2^{k/2}T_j^{(k\vee2)/2}}} +1 \Big) \\
    &\lesssim 
    \begin{cases}
        \frac{\sqrt{d}M^{3/2}\polylog n}{\sqrt{n}}\tau^{-\frac{k}{4}+\frac{1}{2}}, \quad &k\geq 2 \\
        \frac{\sqrt{d}M^{3/2}\polylog n}{\sqrt{n}}, & k=1
    \end{cases}.
\end{align*}

Further take $\delta = n^{-1}$ and take $\tau = n^{-2/k}$, then it holds that, 
\begin{align*}
    W_1(p^\star,p_{\widehat{Y}_{T-\tau}}) \lesssim \frac{dM^{3/2} \polylog n}{n^{1/(k\vee 2)}}.
\end{align*}

%% file: pf-of-lemmas.tex
\section{Proof of Lemmas}
\label{sec:pf-of-lemmas}

As long as we could ensure that with high probability the data segmentation with $n_0$ points are correct and each class contains at least $k+1$ points, then we could recover the exact linear subspaces $V_j$ and $c_j(X) = \ind_{\{ X \in V_j\}}$ almost surely for $X \sim p^\star$. This gives the following Lemma \ref{lem:subspace clustering}. 
\begin{lemma}[Subspace Clustering]
    \label{lem:subspace clustering}
    Under Assumption \ref{assume:multi-modal}, there exists an algorithm that uses the upper bound of $M$ and $k$ and ensures exact recovery of linear subspaces and hence the function $c_j$ with high probability. That is, for any large $n$, define
    \begin{align*}
        \mathcal{E} \defn \Big\{ \forall i \in [M], \exists j_i \in [M] \text{ s.t. } V_{j_i}  = \widehat{V}_i \Big\}
    \end{align*}
    Here $\widehat{V}_i$ denotes the $i$-th estimated subspace of this algorithm $\mathsf{Alg}(n_0)$. Then 
    \begin{align*}
        \mathbb{P} \big[ \mathcal{E}^c\big] \lesssim Mn^{-10}
    \end{align*}
    with the randomness taken over samples $\{ X^{(i)}\}_{i=1}^{n_0}$, and here we take $n_0 = O(c_{p}^2M^2(k+1)\log n)$. 
\end{lemma}
We will further implement an algorithm that proves Lemma \ref{lem:subspace clustering} and discuss its complexity in Appendix \ref{pf-of-lemma:lem:subspace clustering}. 
This lemma is just for theoretical guarantee for our further score estimation under exact subspaces recovery. 

With exact subspace recovery, when constructing the score component in (\ref{eq:score decomp}), we only need to estimate $s_t^\low$ 
and the following lemma gives a time-dependent rate that only relies on intrinsic dimension $k_i$, using certain low-dimensional estimator (\ref{eq:low-dimensional score estimator}).
\begin{lemma}[Low-dim score estimator]
    \label{lem:low-dim score estimator}
    For any $\sigma$ sub-gaussian distribution $\nu$ in $\mathbb{R}^k$, denote $p_t$ as the density of $\nu \star \mathcal{N}(0,tI_k)$ and $s_t = \nabla \log p_t$ as its score. Suppose that $\{ X^{(i)}\}_{i=1}^N$ are $N$ i.i.d samples from $\nu$. 
    Then for any $t>0$, we could construct a kernel-based score estimator $\widehat{s}_t$ using (\ref{eq:low-dimensional score estimator}) that satisfies, 
    \begin{enumerate}
        \item \textbf{Time-dependent $L^2$ estimation error.}
        \begin{align*}
            \int_{\mathbb{R}^k} \mathbb{E}\big[\big\| \widehat{s}_t(x) - s_t(x)\big\|_2^2 \big] p_t(x) \diff x \lesssim  \Big( \frac{4}{\sqrt{\pi}}\Big)^{k} \frac{1}{N} \Big( \frac{1}{t} + \frac{\sigma^k }{ t^{k/2+1}}\Big) (\log N)^{k/2+2} 
        \end{align*}
        with expectation taken over samples $\{ X^{(i)}\}_{i=1}^N$. 
        \item \textbf{Bounded estimator}.
        \begin{align*}
            \| \widehat{s}_t(x) \| _2 \leq \sqrt{\frac{2}{t}\log N}
        \end{align*}
    \end{enumerate}
\end{lemma}
Here we adopt the estimator from \citet{cai2025minimaxoptimalityprobabilityflow} with a further cut-off. 
This achieves a similar rate w.r.t $t$ compared with \citet{zhang2024minimaxoptimalityscorebaseddiffusion,cai2025minimaxoptimalityprobabilityflow}. The proof is provided in Appendix \ref{pf-of-lem:lem:low-dim score estimator}.

\subsection{Proof of Lemma \ref{lem:subspace clustering}.}
\label{pf-of-lemma:lem:subspace clustering}
Basically, subspace clustering for noise-free model aims to solve the following optimization problem iteratively, 
\begin{align*}
    \min_{A} \sum_{i=1}^{n_0} \| X^{(i)}-AA^{\top}X^{(i)}\|_0
\end{align*}
Here the $L_0$ norm is defined as, 
\begin{align*}
    \| x\|_0 = 
    \begin{cases}
        0 , \quad x=0\\
        1, \quad \text{else}
    \end{cases}
\end{align*}
As a non-convex and non-smooth optimization problem, it is basically an NP-hard problem. 

With known upper bound of both intrinsic dimension $k$ and the number of subspaces $M$, the following algorithm is guaranteed to recover exact subspace recovery with high probability and under Assumption \ref{assume:multi-modal}, 
\begin{algorithm}[h]
\caption{Exact subspace recovery}
\label{alg:exact subspace recovery}
\begin{algorithmic}[1]
\REQUIRE Samples $\{X^{(i)}\}_{i=1}^{n_0}$, upper bound of the number of subspaces $\widetilde{M}\geq M$, upper bound of the intrinsic dimension $\widetilde{k}\geq k$
\FOR{j=1,..,$\widetilde{M}$}
    \STATE Iterate over all separations of remained samples into $2$ categories with one having $\widetilde{k}+1$ samples until finding a case that these $\widetilde{k}+1$ samples are linearly dependent.
    \STATE Find the smallest $p$, such as there exists $p+1$ samples from these $\widetilde{k}+1$ samples that are linearly dependent.
    \STATE Define $\widehat{V}_j $ as the span of these $p+1$ points. 
    \STATE Delete those samples that are on $\widehat{V}_j$.
\ENDFOR
\end{algorithmic}    
\end{algorithm}

Define the event sets, 
\begin{align*}
    \mathcal{E}_0 &\defn \Big\{ \exists \text{ at least } k+1 \text{ samples on each subspace}\Big\} \\
    \mathcal{E}_1(p) & \defn \Big\{ x \in \text{span} \{ y_1.\cdots,y_p\}, V_{j} \nsubseteq\text{ span} \{ y_1,\cdots,y_p\}, \forall j\in [M] \Big\}, \quad \text{for }p\leq \widetilde{k}
\end{align*}
To ensure that Algorithm \ref{alg:exact subspace recovery} works with probability larger than $1-Mn^{-10}$, we only need, 
\begin{subequations}
    \begin{align}
    \mathbb{P}[\mathcal{E}_0^c] &\lesssim Mn^{-10}\label{eq:low prob 1} \\
    \mathbb{P}_{x,y_1,\cdots y_p \overset{\text{i.i.d}}{\sim} \mu^*}[\mathcal{E}_1(p)] &= 0, \quad \forall p\leq \widetilde{k} \label{eq:low prob 2}
    \end{align}
\end{subequations}
Since under $\mathcal{E}_1(p)^c$, samples from other subspaces will be excluded via finding the smallest $p$, we could recover the exact subspace in each iteration.

\paragraph*{Proof of (\ref{eq:low prob 2}).}
We first conditioned on $y_1,\cdots,y_p$ and $V_j \nsubseteq\text{ span} \{ y_1,\cdots,y_p\}$ for all $j\in [M]$, then, 
\begin{align*}
    \mathbb{P} [X\in \text{span} \{ y_1.\cdots,y_p\}] &= \sum_{j=1}^M \int_{V_j \cap \text{span} \{ y_1.\cdots,y_p\} }  p_j^\star(\!\diff x) \\
    &= 0 \quad (\text{Assumption \ref{assume:multi-modal}})
\end{align*}
Then we integrate this over $y_1,\cdots,y_p \sim \mu^*$ and get $\mathbb{P}[\mathcal{E}_1(p)] = 0$.

\paragraph*{Proof of (\ref{eq:low prob 1}).} 
Basically, under Assumption \ref{assume:multi-modal}, denote $N_i(n_0), \forall i \in [M]$ as the random variable of the number of samples on $V_i$ with sample size $n_0$. Then, we could view $N_i(n_0)$ as the sum of Bernoulli r.v with $p=p_i^\star(V_i)\geq \frac{1}{c_{\mu}M}$ and thus, 
\begin{align*}
    N_i(n_0) \geq \frac{n_0}{2c_{\mu}M}, \quad \text{with probability }\geq 1-2e^{-\frac{n_0}{2c_{\mu}^2M^2}}
\end{align*}
here we apply Hoeffding inequality just like the proof of Claim \ref{claim:equal separation}. Finally, we could apply union bound and take, 
\begin{align*}
    n_0 = O(c_{p}^2M^2(k+1) \log n) .
\end{align*}

The iteration complexity for Algorithm \ref{alg:exact subspace recovery} is bounded by $\begin{pmatrix}
    n_0\\
    \widetilde{k}+1
\end{pmatrix}\cdot M$.

\subsection{Proof of Lemma \ref{lem:MSE for weight}.}
\label{pf-of-lem:lem:MSE for weight}
The following lemma helps prove this result, 
\begin{lemma}[MSE for $\widehat{p}_t(x)$ and $\widehat{q}_t(i,x)$]
    \label{lem:MSE for p_t and p_t(i)}
    For kernel-based estimators $\widehat{p}_t(x)$ and $\widehat{q}_t(i,x)$ in (\ref{eq:estimator for p_t(x)}), (\ref{eq:estimator for p_t(i,x)}), it holds that, 
    \begin{enumerate}
        \item They are unbiased estimators for $p_t(x)$ and $q_t(i,x)$ respectively under Assumption \ref{assume:multi-modal}, i.e., 
        \begin{subequations}
            \begin{align*}
                \mathbb{E}[\widehat{p}_t(x)] &= p_t(x) \\
                \mathbb{E}[\widehat{q}_t(i,x)] &= q_t(i,x), \quad \forall i \in [M]
            \end{align*}
        \end{subequations}

        \item Under Assumption \ref{assume:multi-modal}, we have the following point-wise MSE bound, 
        \begin{subequations}
            \begin{align}
                \mathbb{E}\big[\big(\widehat{p}_t(x)-p_t(x)\big)^2\big]&\leq \frac{1}{(2\pi t)^{d/2} N} \sum_{i=1}^M  e^{-\frac{1}{2t} \| x-\proj_i(x)\|_2^2} \cdot q_t(i,x)\label{eq:MSE for p_t} \\
                \mathbb{E}\big[\big(\widehat{q}_t(i,x)-q_t(i,x)\big)^2\big] &\leq \frac{1}{(2\pi t)^{d/2} N} e^{-\frac{1}{2t} \| x-\proj_i(x)\|_2^2} q_t(i,x), \quad \forall i \in [M]\label{eq:MSE for p_t(i)}
            \end{align}
        \end{subequations}
    \end{enumerate}
\end{lemma}

We leave the proof of this lemma in the end of this section, and we first prove Lemma \ref{lem:MSE for weight} as follows using Lemma \ref{lem:MSE for p_t and p_t(i)}. 
For $x\in \mathcal{G}_t(i)$
\begin{align*}
    \mathbb{E}\big[\big(w_t(i,x)-\widehat{w}_t(i,x) \big)^2\big] &= \mathbb{E} \Big[\big(\frac{q_t(i,x)}{p_t(x)} - \frac{\widehat{q}_t(i,x)}{\widehat{p}_t(x)}\big)^2 \Big] \\
    &\lesssim \mathbb{E} \Big[ \big(\frac{q_t(i,x)-\widehat{q}_t(i,x)}{p_t(x)}\big)^2\Big] + \mathbb{E} \Big[ \big( \frac{\widehat{q}_t(i,x)}{\widehat{p}_t(x)}\big)^2 \cdot \frac{(\widehat{p}_t(x)-p_t(x))^2}{p_t^2(x)}\Big] 
\end{align*}
\begin{itemize}
    \item For the first term, 
    \begin{align*}
        \mathbb{E} \Big[ \big( \frac{q_t(i,x)-\widehat{q}_t(i,x)}{p_t(x)}\big)^2\Big] &\leq \frac{1}{p_t^2(x)} \cdot \frac{1}{N}(2\pi t)^{-d/2} e^{-\frac{1}{2t} \| x-\proj_i(x)\|_2^2} q_t(i,x). \quad (\text{Lemma \ref{lem:MSE for p_t and p_t(i)}})
    \end{align*}

    \item For the second term, 
    \begin{align*}
        \mathbb{E} \Big[ \big( \frac{\widehat{q}_t(i,x)}{\widehat{p}_t(x)}\big)^2 \cdot \frac{(\widehat{p}_t(x)-p_t(x))^2}{p_t^2(x)}\Big]  &\leq \mathbb{E} \Big[ \big( \frac{\widehat{p}_t(x)-p_t(x)}{p_t(x)}\big)^2\Big] \\
        &\leq \frac{1}{p_t^2(x)} \cdot \frac{1}{(2\pi t)^{d/2}N} \Big(\sum_{i=1}^M e^{-\frac{1}{2t}\|x-\proj_i(x)\|_2^2} \cdot q_t(i,x) \Big) \quad (\text{Lemma \ref{lem:MSE for p_t and p_t(i)}})
    \end{align*}
\end{itemize}
Therefore, it holds that, 
    \begin{align*}
        \mathbb{E}[\big(w_t(i,x)-\widehat{w}_t(i,x) \big)^2] \lesssim \frac{1}{p_t^2(x)} \cdot \frac{1}{(2\pi t)^{d/2}N} \Big(\sum_{i=1}^M e^{-\frac{1}{2t}\|x-\proj_i(x)\|_2^2} \cdot q_t(i,x) \Big) .
    \end{align*}

Then we prove Lemma \ref{lem:MSE for p_t and p_t(i)}, and the proof idea is quite similar to the proof of Lemma 3 in \citet{cai2025minimaxoptimalityprobabilityflow}. 

\paragraph*{Proof of Lemma \ref{lem:MSE for p_t and p_t(i)}.}
\begin{itemize}
    \item First prove that both are unbiased estimators, 
    \begin{align*}
        \mathbb{E}[\widehat{p}_t(x)] &= \mathbb{E}[\frac{1}{N}\sum_{j=1}^N \varphi_t(x-X^{(j)};d)] = \mathbb{E}_{X\sim p^\star}[\varphi_t(x-X;d)] = p_t(x) \\
        \mathbb{E}[\widehat{q}_t(i,x)] &= \mathbb{E}_{X\sim p^\star} [\varphi_t(x-X;d) \ind_{\{ c(X) = i\}}] = \mathbb{E}_{X\sim p^\star} [\varphi_t(x-X;d) \ind_{\{ X \in V_i\}}] = \int_{V_i} \varphi_t(x-y;d) p_i^\star(\!\diff y)
    \end{align*}
    In the last equation, we apply Assumption \ref{assume:multi-modal}. 

    \item Then derive the mean squared error (MSE) for both estimators, i.e., both variance. To prove (\ref{eq:MSE for p_t}), observe that,
    \begin{align*}
        \var(\widehat{p}_t(x)) &= \frac{1}{N} \var_{X \sim p^\star} (\varphi_t(x-X;d)) \leq \frac{1}{N} \mathbb{E}_{X\sim p^\star}[\varphi_t^2(x-X;d)] = \frac{1}{N}\sum_{i=1}^M \int_{V_i} \varphi^2_t(x-y;d) p_i^\star(\!\diff y) .
    \end{align*}
    Plug in the structure of $V_i$: 
    \begin{align*}
        \int_{V_i} \varphi_t^2(x-y;d)p_i^\star(\!\diff y) &= \int_{V_i} (2\pi t)^{-d} e^{-\frac{1}{t}\|x-y\|_2^2} p_i^\star(\!\diff y) \\
        &= (2\pi t)^{-d} e^{-\frac{1}{t} \| x-\proj_i(x)\|_2^2}  \int_{V_i} e^{-\frac{1}{t} \| y-\proj_i(x)\|_2^2}  p_i^\star(\!\diff y) 
    \end{align*}
    Notice that, 
    \begin{align*}
        \int_{V_i} e^{-\frac{1}{t} \| y-\proj_i(x)\|_2^2}  p_i^\star(\!\diff y) \leq  \int_{V_i} e^{-\frac{1}{2t} \| y-\proj_i(x)\|_2^2}  p_i^\star(\!\diff y) .
    \end{align*}
    Hence, 
    \begin{align*}
        \int_{V_i} \varphi_t^2(x-y;d)p_i^\star(\!\diff y) \leq (2\pi t)^{-d/2} e^{-\frac{1}{2t} \| x-\proj_i(x)\|_2^2} \cdot q_t(i,x) .
    \end{align*}
    Therefore, 
    \begin{align*}
        \mathbb{E}\big[\big(p_t(x)-\widehat{p}_t(x)\big)^2\big] &\leq \frac{1}{N}\sum_{i=1}^M \int_{V_i} \varphi_t^2(x-y;d)p_i^\star(\!\diff y) \\
        &\leq \frac{1}{(2\pi t)^{d/2}N} \sum_{i=1}^M  e^{-\frac{1}{2t} \| x-\proj_i(x)\|_2^2} \cdot q_t(i,x).
    \end{align*}

    \item To prove (\ref{eq:MSE for p_t(i)}), similarly, 
    \begin{align*}
        \mathbb{E}\big[\big( q_t(i,x) - \widehat{q}_t(i,x)\big)^2\big] \leq \frac{1}{N} \mathbb{E}_{X\sim p^\star} [\varphi^2_t(x-X;d)\ind_{\{ X \in V_i\}} ] = \frac{1}{N} \sum_{j=1}^M \int_{V_j} \varphi_t^2(x-y;d) \ind_{\{ y \in V_i\}} p_j^\star(\!\diff y)
    \end{align*}
    Under Assumption \ref{assume:multi-modal}, it holds that,
    \begin{align*}
        \mathbb{E}\big[\big( q_t(i,x) - \widehat{q}_t(i,x)\big)^2\big] \leq \frac{1}{N} \int_{V_i} \varphi_t^2(x-y;d)  p_i^\star(\!\diff y) &\leq \frac{1}{N}(2\pi t)^{-d} e^{-\frac{1}{t} \| x-\proj_i(x)\|_2^2}  \int_{V_i} e^{-\frac{1}{t} \| y-\proj_i(x)\|_2^2}  p_i^\star(\!\diff y) \\
        & \leq \frac{1}{N}(2\pi t)^{-d} e^{-\frac{1}{t} \| x-\proj_i(x)\|_2^2}  \int_{V_i} e^{-\frac{1}{2t} \| y-\proj_i(x)\|_2^2}  p_i^\star(\!\diff y)  \\
        &\leq \frac{1}{N}(2\pi t)^{-d/2} e^{-\frac{1}{2t} \| x-\proj_i(x)\|_2^2} q_t(i,x).
    \end{align*}
\end{itemize}


\subsection{Proof of Lemma \ref{lem:low-dim score estimator}.}
\label{pf-of-lem:lem:low-dim score estimator}
In the proof of this lemma, the target distribution $\nu$ is a $\sigma$-subgaussian distribution in $\mathbb{R}^k$,
and we denote $p_t$ as the density of $\nu \ast \mathcal{N}(0,tI_k)$ and $s_t(\cdot) \defn \nabla \log p_t(x)$ as its score.
$\{X^{(i)}\}_{i=1}^N \overset{\text{i.i.d}}{\sim} \nu $ are N samples for score estimation.

We take the estimator based on (\ref{eq:low-dimensional score estimator}) as,
    \begin{align}
        \widehat{s}_t(x) = \clip_R\Big( \frac{\nabla \widehat{g}_t(x)}{\widehat{g}_t(x)} \psi\Big(\widehat{g}_t(x);\frac{\log N}{N(2\pi t)^{k/2}}\Big)\Big) \quad \text{with } R \defn \sqrt{\frac{2}{t} \log N}. \label{eq:low-dim score estimator}
    \end{align}
    Here the clip operator is defined as, 
    \begin{align*}
    \clip_r(z) \defn 
    \begin{cases}
        z, \quad &\|z\|_2 \leq  r \\
        \proj_{B_r(0)}(z) ,\quad &\text{else}
    \end{cases}
    \end{align*}
    where $\proj_{B_r(0)}(z) = r \cdot \frac{z}{\|z\|_2}$ means the projection of $z$ on that ball; $\psi(x;\eta) \defn \ind_{\{x\geq \eta \}}$ is the hard thresholding function; 
    $\widehat{g}_t(x) \defn \frac{1}{N}\sum_{i=1}^N \varphi_t(X^{(i)}-x;k)$ is the kernel based density estimator.
    As a remark, the estimator before clipping is almost the same as that in \citet{cai2025minimaxoptimalityprobabilityflow}, 
    and we apply hard-thresholding here for simplicity since we use DDPM sampling procedure.

Then we show that the score estimator defined in (\ref{eq:low-dim score estimator}) satisfies the conditions in Lemma \ref{lem:low-dim score estimator}. 

In \citet{cai2025minimaxoptimalityprobabilityflow}, they defined
\begin{align*}
    \mathcal{F}_t \defn \Big\{ x: p_t(x)\geq \frac{c_{\eta} \log N}{N(2\pi t)^{k/2}}\Big\}
\end{align*}
and it holds that, 
\begin{subequations}
    \begin{align}
    \int_{\mathcal{F}_t^c}  p_t(x) \diff x &\lesssim \Big( \frac{16}{\pi}\Big)^{k/2} \frac{1}{N} \Big( 1 + \frac{\sigma^k }{ t^{k/2}}\Big) (\log N)^{k/2+1} \label{eq:prob bdd on F_t^c} \\
    \int_{\mathcal{F}_t^c} \|s_t(x)\|_2^2\cdot  p_t(x) \diff x &\lesssim \Big( \frac{16}{\pi}\Big)^{k/2} \frac{1}{N} \Big( \frac{1}{t} + \frac{\sigma^k }{ t^{k/2+1}}\Big) (\log N)^{k/2+2} \label{eq:score bdd on F_t^c}
\end{align}
\end{subequations}
from Lemma 8 in \citet{cai2025minimaxoptimalityprobabilityflow}.
Notice that from Lemma 4 in \citet{cai2025minimaxoptimalityprobabilityflow}, it has been proven that, 
\begin{align*}
    \| s_t(x)\|_2^2 \leq \frac{2}{t} \log \frac{1}{(2\pi t)^{k/2} p_t(x)} 
\end{align*}
Hence, for $x\in \mathcal{F}_t$, one has, 
\begin{align}
    \| s_t(x)\|_2^2 \leq \frac{2}{t} \log \frac{N}{c_{\eta} \log N} \leq \frac{2}{t} \log N . \label{eq:score bdd on F_t}
\end{align}
Here we use the constant selection of $c_{\eta} \geq 2$ as in \citet{cai2025minimaxoptimalityprobabilityflow}. 

That's the reason why we use the clip operator in (\ref{eq:low-dim score estimator}). Therefore, with the estimator constructed in (\ref{eq:low-dim score estimator}), it automatically satisfies the bounded condition and for its $L^2$ error, 
\begin{align*}
    \int_{\mathbb{R}^k} \mathbb{E}\big[\big\| \widehat{s}_t(x) - s_t(x)\big\|_2^2 \big] p_t(x) \diff x&\leq \int_{\mathcal{F}_t} \mathbb{E}\big[\big\| \widehat{s}_t(x) - s_t(x)\big\|_2^2 \big] p_t(x) \diff x \\
    &+ \int_{\mathcal{F}_t^c} \mathbb{E}\big[\big\| \widehat{s}_t(x) - s_t(x)\big\|_2^2 \big] p_t(x) \diff x \\
    &\defnrev \chi_1 + \chi_2
\end{align*}

\begin{itemize}
    \item \textbf{For $\chi_1$.}
    Due to (\ref{eq:score bdd on F_t}), for $x\in \mathcal{F}_t$, $s_t(x) = \clip_R(s_t(x))$. What's more, $\clip_R(\cdot)$ is $1$-Lip continous, and thus, 
    \begin{align*}
        \chi_1 &\leq \int_{\mathcal{F}_t}\mathbb{E}\big[\big\|\frac{\nabla \widehat{g}_t(x)}{\widehat{g}_t(x)} \psi\Big(\widehat{g}_t(x);\frac{\log N}{N(2\pi t)^{k/2}}\Big) -s_t(x)\big\|_2^2 \big]  \cdot p_t(x) \diff x \\
        & \leq \int_{\mathbb{R}^k}\mathbb{E}\big[\big\|\frac{\nabla \widehat{g}_t(x)}{\widehat{g}_t(x)} \psi\Big(\widehat{g}_t(x);\frac{\log N}{N(2\pi t)^{k/2}}\Big) -s_t(x)\big\|_2^2 \big]  \cdot p_t(x) \diff x \\
        &\leq \frac{(4/\sqrt{\pi})^k}{N} \Big( \frac{1}{t}+\frac{\sigma^k}{t^{k/2+1}}\Big) (\log N)^{k/2+1}
    \end{align*}
    The last inequality results from Proposition 1 in \citet{cai2025minimaxoptimalityprobabilityflow}, which holds for hard thresholding $\psi$ by simple modification.

    \item \textbf{For $\chi_2$.}
    Notice that, 
    \begin{align*}
        \chi_2 &\defn \int_{\mathcal{F}_t^c}\mathbb{E}\big[\big\| \widehat{s}_t(x) - s_t(x)\big\|_2^2 \big] p_t(x) \diff x \\
        &\leq \int_{\mathcal{F}_t^c} \Big( R^2 + \|s_t(x)\|_2^2\Big) p_t(x) \diff x \\
        &\leq R^2 \int_{\mathcal{F}_t^c} p_t(x) \diff x + \int_{\mathcal{F}_t^c}\|s_t(x)\|_2^2 \cdot  p_t(x) \diff x \\
        &\lesssim \Big( \frac{16}{\pi}\Big)^{k/2} \frac{1}{N} \Big( \frac{1}{t} + \frac{\sigma^k }{ t^{k/2+1}}\Big) (\log N)^{k/2+2} \quad (\text{Using }(\ref{eq:prob bdd on F_t^c}) + (\ref{eq:score bdd on F_t^c}))
    \end{align*}
\end{itemize}

Therefore, it holds that, 
\begin{align*}
    \int_{\mathbb{R}^k} \mathbb{E}\big[\big\| \widehat{s}_t(x) - s_t(x)\big\|_2^2 \big] p_t(x) \diff x \lesssim \Big( \frac{4}{\sqrt{\pi}}\Big)^{k} \frac{1}{N} \Big( \frac{1}{t} + \frac{\sigma^k }{ t^{k/2+1}}\Big) (\log N)^{k/2+2} .
\end{align*}

\subsection{Proof of Lemma \ref{lem:bound of s_t(i,x)}}
\label{pf-of-lem:lem:bound of s_t(i,x)}
Observe that, conditioned on $N_i $, we could obtain $N_i$ i.i.d samples from normalized distribution $p_i^\star$ to estimate its score function.   
        We only need to prove (\ref{eq:bdd of s_t(i,x)}). Notice that, 
        \begin{align*}
            q_t(i,x) = \int_{V_i} \varphi_t(x-y;d) p_i^\star(\!\diff y) &= (2\pi t)^{-d/2} e^{-\frac{1}{2t}\| x-\proj_i(x)\|_2^2} \int_{V_i} e^{-\frac{1}{2t}\| \proj_i(x)-y\|_2^2} p_i^\star( \!\diff y) \\
            &= (2\pi t)^{-d/2} e^{-\frac{1}{2t}\| x-\proj_i(x)\|_2^2} p_i^\star(V_i) \int_{\mathbb{R}^k} e^{-\frac{1}{2t}\| z - A_i^{\top}x\|_2^2} p_i^{\low}(\! \diff z) \\
            &= (2\pi t)^{-(d-k_i)/2} e^{-\frac{1}{2t}\| x-\proj_i(x)\|_2^2} p_i^\star(V_i) \cdot p_t^{\low}(i,A_i^{\top}x)
        \end{align*}
        here $p_t^{\low}(i,\cdot)$ denotes the density function of $p_i^{\low} \ast \mathcal{N}(0,tI_{k_i})$ in $\mathbb{R}^{k_i}$. 
        Hence, 
        \begin{align*}
            &\int_{\mathbb{R}^d} \mathbb{E}\big[\big\| \widehat{s}_t(i,x) - s_t(i,x)\big\|_2^2 \big| N_i \big] q_t(i,x)  \diff x \\
            &\leq p_i^\star(V_i)\int_{\mathbb{R}^d} \mathbb{E}\big[\big\| \widehat{s}_t^{\low}(i,A_i^{\top}x) - s_t^{\low}(i,A_i^{\top}x)\big\|_2^2\big| N_i \big] (2\pi t)^{-(d-k_i)/2} e^{-\frac{1}{2t}\| x-A_iA_i^{\top}x\|_2^2}  \cdot p_t^{\low}(i,A_i^{\top}x)  \diff x
        \end{align*}
        Apply the same linear transform as in (\ref{eq:changing var}), 
    \begin{align*}
        &\int_{\mathbb{R}^d} \mathbb{E}\big[\big\| \widehat{s}_t(i,x) - s_t(i,x)\big\|_2^2 \big| N_i \big] q_t(i,x)  \diff x \\ 
        &\leq p_i^{\star}(V_i) \int_{\mathbb{R}^d} \mathbb{E}\big[\big\| \widehat{s}_t^{\low}(i,z_{1:{k_i}}) - s_t^{\low}(i,z_{1:{k_i}})\big\|_2^2 \big| N_i \big] (2\pi t)^{-(d-k_i)/2} e^{-\frac{1}{2t}\| z_{k_i+1:d}\|_2^2}  \cdot p_t^{\low}(i,z_{1:k_i})  \diff z\\
        &= p^\star(V_i) \int_{\mathbb{R}^{k_i}} \mathbb{E}\big[\big\| \widehat{s}_t^{\low}(i,z_{1:k_i}) - s_t^{\low}(i,z_{1:k_i})\big\|_2^2  \big| N_i \big]   \cdot p_t^{\low}(i,z_{1:k_i})  \diff z_{1:k_i} \quad (\text{Tonelli's Theorem})\\
        &\lesssim p_i^{\star}(V_i) \cdot \frac{(4/\sqrt{\pi})^{k_i}}{N_i}\Big( \frac{1}{t}+\frac{\sigma^{k_i}}{t^{k_i/2+1}}\Big) \cdot \polylog N \quad (\text{Lemma \ref{lem:low-dim score estimator}})
    \end{align*}
    Finally, 
    \begin{align*}
        \int_{\mathbb{R}^d} \mathbb{E}\big[\big\| \widehat{s}_t(i,x) - s_t(i,x)\big\|_2^2 \ind_{\{N_i \geq n_i \}}\big] q_t(i,x) \diff x &= \int_{\mathbb{R}^d} \mathbb{E}\Big[\mathbb{E}\big[\big\| \widehat{s}_t(i,x) - s_t(i,x)\big\|_2^2 \big| N_i \big] \cdot \ind_{\{N_i \geq n_i \}}\Big] q_t(i,x) \diff x\\
        & = \mathbb{E}\Big[ \Big( \int_{\mathbb{R}^d} \mathbb{E}\big[\big\| \widehat{s}_t(i,x) - s_t(i,x)\big\|_2^2 \big| N_i \big] \cdot q_t(i,x) \diff x\Big) \cdot  \bm{1}_{\{N_i \geq n_i \}}\Big] \\
        &\lesssim p_i^\star(V_i) \cdot \frac{(4/\sqrt{\pi})^{k_i}}{n_i}\Big( \frac{1}{t}+\frac{\sigma^{k_i}}{t^{k_i/2+1}}\Big) \cdot \polylog N .
    \end{align*}

%% file: pf-of-claims.tex
\section{Proof of Claims}
\subsection{Proof of Claim \ref{claim:equal separation}}
\label{pf-of-claim:claim:equal separation}
For any $i\in[M]$, we first bound the probability $\mathbb{P}[N_i < \frac{N}{2c_{p}M}]$. 
Note that $N_i = \sum_{j=1}^{N} \ind_{\{ X^{(j)} \in V_i\}}$ is basically the sum of $N$ i.i.d Bernoulli r.vs with parameter $p^\star(V_i) \geq \frac{1}{c_{p}M}$ (Assumption \ref{assume:multi-modal}). 
Applying Hoeffding's inequality, we have,
\begin{align*}
    \mathbb{P}\Big[ N_i < \frac{N}{2c_{p}M} \Big] & = \mathbb{P}\Big[ N_i - N\cdot p^\star(V_i) < -\Big(N\cdot p^\star(V_i) - \frac{N}{2c_{p}M}\Big) \Big] \\
    &\leq \exp \Big( -2N\cdot (\frac{1}{2c_{p}M})^2 \Big) = \exp \Big( -\frac{N}{2c_{p}^2M^2}\Big).
\end{align*}
Then applying union bound over $i\in[M]$, it holds that,
\begin{align*}
    \mathbb{P}\big[\mathcal{A}^c\big] = \mathbb{P}\Big[\exists i \in [M], N_i < \frac{N}{2c_{p}M}\Big] \leq Me^{-\frac{N}{2c_{p}^2M^2}}.
\end{align*}

\subsection{Proof of Claim \ref{claim:bdd of set B_t}}
\label{pf-of-claim:claim:bdd of set B_t}

For the target distribution $p^\star$, it is supported on $\cup_{i=1}^M V_i$. 
Hence for any $\theta \in \mathbb{R}^d$ and $\| \theta\|_2 =1$, denote r.v $X \sim p^\star$
\begin{align*}
    \mathbb{E}\Big[ \exp \Big( \frac{(X^{\top}\theta)^2}{\sigma^2}\Big)\Big] &= \sum_{i=1}^M \int_{V_i} \exp \Big( \frac{(x^{\top}\theta)^2}{\sigma^2} \Big) p_i^\star(\!\diff x) \\
    & = \sum_{i=1}^M p_i^\star(V_i) \int_{\mathbb{R}^{k_i}}\exp \Big( \frac{(z^{\top}A_i^{\top}\theta)^2}{\sigma^2} \Big) p_i^{\low}(\!\diff z) \\
    & \leq \sum_{i=1}^M p_i^\star(V_i) \int_{\mathbb{R}^{k_i}}\exp \Big( \frac{(z^{\top}A_i^{\top}\theta)^2}{\sigma_i^2} \Big) p_i^{\low}(\!\diff z) \\
    &\leq \sum_{i=1}^M p_i^\star(V_i) \cdot 2 = 2.
\end{align*}
Here we use $\| A_i^{\top}\theta \| \leq 1$ and Assumption \ref{assump:sub-gaussian target}. 
This shows that, $p^\star$ is a sub-gaussian distribution in $\mathbb{R}^d$ with parameter $\sigma$. 
Therefore, it is straightforward that $p^\star \ast \mathcal{N}(0,tI_d)$ is $c\sqrt{\sigma^2 +t}$ sub-gaussian, for some absolute constant $c>0$. 

\begin{itemize}
    \item For $\int_{\mathcal{B}_t^c} p_t(x)\!\diff x$, it holds that, 
    \begin{align*}
        \int_{\mathcal{B}_t^c} p_t(x)\!\diff x &\leq \sum_{i=1}^M \bP_{Z_t \sim p_t} \big[\|A_i^{\top}Z_t\|_2>B_t\big] \\
        & \leq \sum_{i=1}^M  \bP_{Z_t \sim p_t} \big[\|A_i^{\top}Z_t\|_{\infty}>B_t/\sqrt{k_i}\big] \lesssim Mk \exp\big( \frac{cB_t^2}{k_i(\sigma^2+t)}\big) \lesssim \frac{Mk}{N^4}\quad  (\text{Lemma \ref{lem:tail bdd for subgaussian}})
    \end{align*}
    Recall the definition of $\mathcal{B}_t$ in (\ref{eq:def of B_t}), and we take large enough constant $C_B>0$ to ensure $\exp\big( \frac{cB_t^2}{k_i(\sigma^2+t)}\big)\lesssim N^{-4}$. 

    \item For $\int_{\mathcal{B}_t^c} \|x\|_2^2\cdot p_t(x)\!\diff x$, it holds that, 
    \begin{align*}
        \int_{\mathcal{B}_t^c} \|x\|_2^2\cdot p_t(x)\!\diff x  &= \mathbb{E}_{Z_t \sim p_t} [\| Z_t \|_2^2 \ind_{\{ Z_t\notin \mathcal{B}_t\}}] \\
        & \leq \sqrt{\bE_{Z_t \sim p_t}[\|Z_t\|_2^4] \cdot \int_{\mathcal{B}_t^c} p_t(x) \!\diff x} \quad (\text{C-S Ineq})\\
        & \lesssim \sqrt{d^2(\sigma^2+t)^2 \frac{Mk}{N^4}}\\
        &\lesssim \frac{d\sqrt{M}}{N^2} (\sigma^2+t).  
    \end{align*}
\end{itemize}

%% file: Auxiliary_lemmas.tex
\section{Auxiliary Lemmas}
\begin{lemma}[Tail bound for subgaussian random vectors]
        \label{lem:tail bdd for subgaussian}
        Let $\nu$ be a $\sigma$-subgaussian distribution in $\mathbb{R}$, i.e, 
    \begin{align}
        \sigma = \| X\|_{\psi_2} \defn \inf \Big\{ t>0: \mathbb{E}\exp \big(X^2/t^2\big) \leq 2\Big\}, \quad X\sim \nu\label{eq:subgaussian def}.
    \end{align}
    Then we have the following tail bound for $X \sim \nu$, any $r\geq 0$ and some absolute constant $c>0$, 
    \begin{align}
        \mathbb{P}[|X| \geq r] &\leq  2\exp\big(-\frac{cr^2}{\sigma^2}\big) \label{eq:subgaussian tail bdd} \\
        \mathbb{E}[|X|^2 \ind_{\{|X|\geq r\}}] & \leq 2 (r^2 + \sigma^2/c) \exp\big(-\frac{cr^2}{\sigma^2}\big) .\label{eq:subgaussian tail bdd of second moment}
    \end{align}
\end{lemma}

\begin{proof}
\begin{itemize}
    \item \textbf{Proof of (\ref{eq:subgaussian tail bdd}).}
    This is a classical result for sub-gaussian r.v. We can find the proof in Proposition 2.5.2 in \cite{Vershynin_2018}. 

\item \textbf{Proof of (\ref{eq:subgaussian tail bdd of second moment}).}
For any $r\geq0$, 
\begin{align*}
    \mathbb{E} \big[ |X|^2\ind_{\{|X|\geq r\}} \big] &= \int_{\mathbb{R}} x^2 \ind_{\{|x|\geq r\}} \cdot p(x) \!\diff x \\
    & = \int_{\mathbb{R}} \Big( \int_0^{|x|} 2z dz\Big) \ind_{\{|x|\geq r\}} p(x) \!\diff x \\
    &= 2\int_{\mathbb{R}}  \int_0^{+\infty} z\cdot \ind_{\{|x|\geq z\}}  \ind_{\{|x|\geq r\}} p(x) \!\diff z \!\diff x \quad (\text{Tonelli's Theorem}) \\
    &= 2\int_0^{+\infty} \Big( \int_{\mathbb{R}}\ind_{\{|x| \geq r\vee z\}}\cdot  p(x) \!\diff x\Big)z \!\diff z \quad (\text{Tonelli's Theorem})\\ 
    & \leq r^2 \mathbb{P}[|X| \geq r] + 2 \int_r^{+\infty} z \cdot \mathbb{P}[|X| \geq z] \!\diff z \\
    &\leq r^2 \mathbb{P}[|X| \geq r] + 4 \int_r^{\infty} z \cdot e^{-cz^2/\sigma^2}  \!\diff z \\
    &= 2 (r^2 + \sigma^2/c) \exp\big(-\frac{cr^2}{\sigma^2}\big) .
\end{align*}
\end{itemize}

\end{proof}